\documentclass[12pt]{article}
\pdfoutput = 1

\usepackage{natbib}
\usepackage[utf8]{inputenc} 
\usepackage[T1]{fontenc}    
\usepackage[unicode=true,psdextra]{hyperref}
\usepackage{url}            
\usepackage{array, booktabs, multirow}       
\usepackage{amsfonts}       
\usepackage{nicefrac}       
\usepackage{microtype}      
\usepackage{xcolor}         

\usepackage{amsmath,bm, amsthm, amssymb}
\usepackage{graphicx,subfig}
\usepackage{algorithm}
\usepackage{algorithmic}
\graphicspath{{figures/}}
\usepackage{subcaption}
\usepackage{arydshln}

\allowdisplaybreaks

\usepackage[margin=2.5cm, bottom=2.5cm]{geometry}

\usepackage{tikz}
\usetikzlibrary{bayesnet}
\usetikzlibrary{arrows}
\usepackage{color}
\usepackage{caption}
\usetikzlibrary{backgrounds}
\usetikzlibrary{shapes.geometric}
\tikzstyle{arrow} = [thick,->,>=stealth]
\usetikzlibrary{shapes.geometric, arrows}

\DeclareCaptionStyle{italic}[justification=centering]{labelfont={bf},textfont={it},labelsep=colon}
\usepackage{mathtools}
\usepackage{todonotes}
\usepackage{array}
\usepackage{thmtools,thm-restate}
\usepackage{enumitem}
\usepackage{tikz-qtree}

\newcolumntype{R}[1]{>{\raggedleft\let\newline\\\arraybackslash\hspace{0pt}}m{#1}}
\newcolumntype{L}[1]{>{\raggedright\let\newline\\\arraybackslash\hspace{0pt}}m{#1}}
\newcolumntype{C}[1]{>{\centering\let\newline\\\arraybackslash\hspace{0pt}}m{#1}}

\renewcommand{\deg}{\text{deg }}

\newcommand{\CP}{\text{CP}}
\newcommand{\Addv}{\text{Add}_v}
\newcommand{\Adde}{\text{Add}_e}
\newcommand{\Delv}{\text{Del}_v}
\newcommand{\Dele}{\text{Del}_e}
\newcommand{\relocatee}{\text{Relocate}_e}
\newcommand{\mergev}{\text{Merge}_v}
\newcommand{\splitv}{\text{Split}_v}

\DeclareMathOperator*{\argmin}{arg\,min}

\newtheorem{theorem}{Theorem}
\newtheorem{lemma}{Lemma}

\newtheorem{proposition}{Proposition}
\newtheorem{definition}{Definition}

\newcommand{\cheng}[1]{}
\newcommand{\SK}[1]{}

\tikzstyle{boxprocess} = [rectangle, rounded corners, text width=4cm,  minimum height=1cm,text centered, draw=black]

\newif\ifshortversion
\shortversiontrue  

\tikzstyle{adjmat} = [rectangle, minimum width=3cm, text centered, draw=black, fill=green!30]
\tikzstyle{adjmat2} = [rectangle, minimum width=2cm, text centered, draw=black, fill=green!30]
\tikzstyle{adjmatsmall} = [rectangle, minimum width=2.5cm, text centered, draw=black, fill=green!30]
\tikzstyle{arrow} = [thick,->,>=stealth]
\tikzstyle{doublearrow} = [thick,<->,>=stealth]
\tikzstyle{thmbox} = [rectangle, minimum width=3cm, text centered, draw=black]

\title{A pseudo-inverse of a line graph}
\author{Sevvandi Kandanaarachchi, Philip Kilby, Cheng Soon Ong }
\date{\today}

\begin{document}

\maketitle

\begin{abstract}
    Line graphs are an alternative representation of graphs where each vertex of the original (root) graph becomes an edge.
However not all graphs have a corresponding root graph, hence the transformation  from graphs to line graphs is not invertible.
We investigate the case when there is a small perturbation in the space of line graphs,
and try to recover the corresponding root graph, essentially defining the inverse of the line graph operation.
We propose a linear integer program that edits the smallest number of edges in the line graph,
that allow a root graph to be found.
We use the spectral norm to theoretically prove that such a pseudo-inverse operation is well behaved.
Illustrative empirical experiments on Erdős-Rényi graphs show that our theoretical results work in practice.
\end{abstract}

\section{Introduction}\label{sec:intro}

Graph perturbations are used to test robustness of algorithms. The expectation is that for small graph perturbations algorithm output should not  change drastically. While graph perturbations are extensively studied in many contexts, they are underexplored for line graphs, where a line graph is an alternative representation of a graph obtained by mapping edges to vertices.  But line graphs are increasingly used in many graph learning tasks including link prediction \citep{cai2021line}, expressive GNNs \citep{yang2024theoretical} and community detection \citep{chen2019supervised}, and in other scientific disciplines \citep{ruff2024connectivity, min2023partition, halldorsson2013estimating}. 
The reason that line graph perturbations are not commonly used is because the perturbed graph may not be a line graph. We introduce a pseudo-inverse of a line graph, which generalises the notion of the inverse line graph extending it to non-line graphs. The proposed pseudo-inverse is computed by minimally modifying the perturbed line graph so that it results in a line graph.

Given a graph $G$, its line graph $L(G)$ is obtained by mapping edges of $G$ to vertices of $L(G)$ and connecting vertices in $L(G)$ if the corresponding edges share a vertex (see Figure \ref{fig:linegraph1}). Suppose we perturb the line graph by adding an edge to it. The key point is that the resulting graph may not be a line graph. This is because there are nine line-forbidden graphs \citep{beineke1970characterizations}, which, if present in the perturbed graph will break the line graph. In this sense, line graphs are very fragile. This makes finding valid line graph perturbations a difficult task. Our contribution of a pseudo-inverse is a step forward in this direction, because given a perturbed graph, by finding a pseudo-inverse line graph, we find a ``close'' graph $\hat{G}$ in the original graph space, which can then be used find a valid perturbed graph by computing $L(\hat{G})$. Furthermore, a pseudo-inverse $\hat{G}$ is useful in applications such as haplotype phasing to estimate the ancestor population size \citep{labbe2021finding} and has some links to the cluster deletion problem \citep{ambrosio2025exact}. 

\begin{figure}[t]
\centering

\begin{tikzpicture}[node distance=1.5cm]
\node (G) [adjmat] {$G$};
\node (H) [adjmat, right of=G, xshift=3cm] {$H = L(G)$};
\node (Htilde) [adjmat, below of=H] {$\widetilde{H} = H+\delta$};
\node (Hhat) [adjmat, below of=Htilde] {$\hat{H} = L(L^\dagger(\widetilde{H}))$};
\node (Ghat) [adjmat, left of=Hhat, xshift=-3cm] {$\hat{G} = L^\dagger(\widetilde{H})$};

\draw [arrow] (G) -- (H);
\draw [arrow] (H) -- (Htilde);
\draw [arrow] (Htilde.west) to [bend right=45] (Ghat.north);
\draw [doublearrow] (Ghat) -- (Hhat);
\end{tikzpicture}

\caption{Setting and notation. Given a graph $G$, we have the corresponding line graph $H:=L(G)$. $\widetilde{H}$ is a distorted version of $H$, which may not be a line graph. $\hat{H}$ is a closest line graph to $\widetilde{H}$, and $\hat{G}$ is a pseudo-inverse of $\widetilde{H}$.}
\label{fig:pseudoinverse1}
\end{figure}
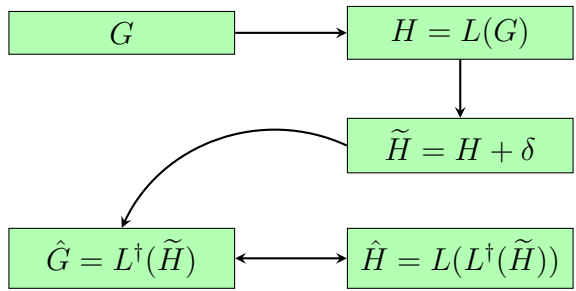

Traditionally, the inverse line graph is called the root graph. However, we use the term inverse line graph instead of root, because we refer to a pseudo-inverse often and the difference between a pseudo-inverse and the inverse is clearer than the difference between a pseudo-inverse and the root. Our contributions can be summarized as follows: 
\begin{enumerate}
    \item We propose a pseudo-inverse of a line graph generalising the inverse line graph to non-line graphs.
    \item Using the spectral radius of the graph adjacency matrix as the norm, we show that for single edge perturbations such a pseudo-inverse is well behaved and bounded.
    \item We propose a linear integer program that finds such a pseudo-inverse, by minimizing edge additions and deletions. 
    \item We illustrate properties of our pseudo-inverse using random graph models and use it as a parent population  estimator for  genotype data.
\end{enumerate}

We provide proof sketches for key theorems in the main text, while all formal proofs are presented in the appendix.


\section{Background and Preliminaries}\label{sec:notation}

Let $G = (V,E)$ denote a graph with vertices $V$ and edges $E$. If $G$ has at least one edge, then its line graph is the graph whose vertices are the edges of $G$, with two of these vertices being adjacent if the corresponding edges share a vertex in $G$ \citep{beineke2021line}. Figure \ref{fig:linegraph1} shows an example of a graph and its line graph.  The edges in the graph on the left are mapped to the vertices in the line graph (on the right) as can be seen from the edge and vertex labels.  
 
\begin{figure}[h]
    \centering
    \includegraphics[width = 0.8\linewidth, trim={0cm 2cm 0 2cm},clip]{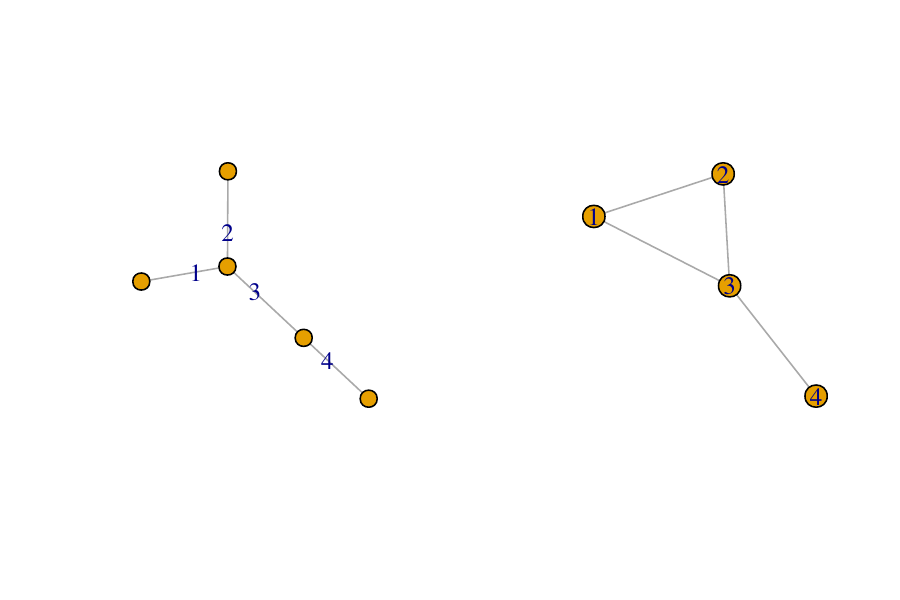}
    \caption{A graph on the left and its line graph on the right. }
    \label{fig:linegraph1}
\end{figure}

We denote the line graph operation by $L$, i.e., for a graph $G$ we denote its line graph by $H:= L(G)$. Then, the inverse line graph is called the root of $H$

\begin{definition}\label{def:linegraphroot}
    If $G$ is a graph whose line graph is $H$, that is, $L(G) = H$, then $G$ is called the \textbf{root} or the \textbf{inverse line graph} of $H$. 
\end{definition}
\cite{whitney1932congruent} showed that the structure of a graph can be recovered from its line graph with one exception: if the line graph $H$ is $K_3$, a triangle, then the root of $H$ can be either  $K_{1,3}$, a star or $K_3$ a triangle. This follows from the following theorem as stated in \cite{harary2018graph}:

\begin{theorem}[Whitney (1932), Harary (1969)]\label{thm:IsomorphicLineGraphs}
    Let $G$ and $G'$ be connected graphs with isomorphic line graphs. Then $G$ and $G'$ are isomorphic unless one is $K_3$ and the other is $K_{1,3}$.
\end{theorem}

By creating edges corresponding to vertices in line graph $H$ and connecting them by merging the vertices if there is an edge between the vertices in $H$ we can obtain the the graph $G$, such that $H = L(G)$.   
Thus, if $H$ is a line graph that it is not $K_3$, then the inverse line graph $L^{-1}(H)$ exists.  

\ifshortversion
\cite{beineke1970characterizations} characterized the space of line graphs in terms of nine excluded graphs.
\begin{theorem}\citep{beineke1970characterizations}\label{thm:linegraphthm}
Let $H= L(G)$ be a line graph. Then none of the nine graphs in Figure \ref{fig:ForbiddenGraphs} is an induced subgraph of $H$. 
\end{theorem}
\fi


\begin{figure}[t]
    \centering
    \includegraphics[width=0.8\linewidth]{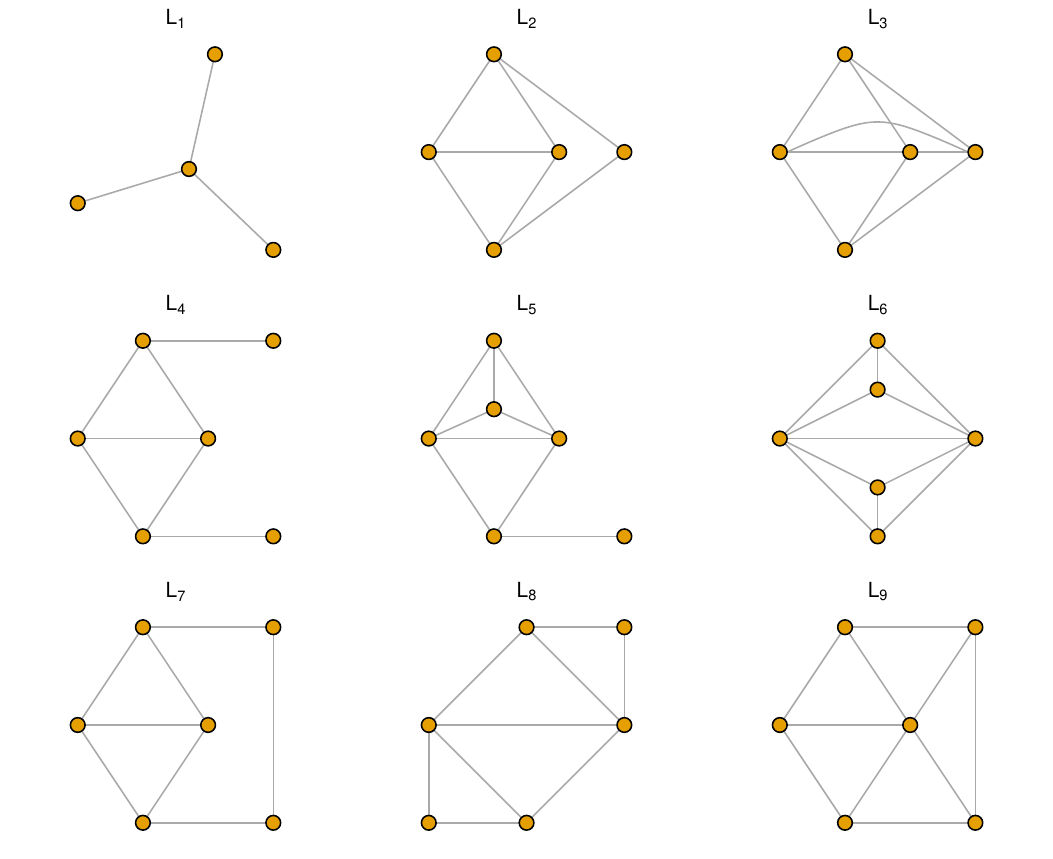}
    \caption{The nine line-forbidden  graphs as illustrated in \cite{beineke2021line}}
    \label{fig:ForbiddenGraphs}
\end{figure}

\ifshortversion
There are several algorithms to find the root of a line graph \citep{ROUSSOPOULOS1973108,Lehot74,Degiorgi95,simic1990algorithm,NAOR1990132,liu2015iligra}. There is also work done on the roots of generalised line graphs \citep{simic1990algorithm}. 

\fi

Our interest is somewhat different. We are interested in line graph perturbations. We use these definitions in the following sections. 



\begin{definition}\label{def:degreenumber}
    Let $G = (V, E)$ be a graph with the vertex set $V$ and the edge set $E$. Let $|V(G)|$ and $|E(G)|$ denote the number of vertices and edges in $G$. Furthermore, let $Z_k(G)$ denote the number of vertices in $G$ with degree $k$. 
\end{definition}

Definitions \ref{def:AddingDeletingEdgesVertices}, \ref{def:RelocateEdge}, \ref{def:MergeVertex} and \ref{def:SplitVertex} define graph edit operations that describe ways a graph $G$ can be modified. We note these are not meant to uniquely identify a graph. For example, we can say a graph $G$ is modified by merging vertices (Definition  \ref{def:MergeVertex}). But the notation does not indicate which vertices merged.

\begin{definition}\label{def:AddingDeletingEdgesVertices}\textbf{(Primary Operations)}
    Let $G$ be a graph. We denote the operations of adding a vertex to $G$, by $\Addv(G)$, adding an edge to $G$ by $\Adde(G)$, deleting an edge from $G$ by $\Dele(G)$ and deleting a vertex from $G$ by $\Delv(G)$. We only use the $\Delv(G)$ operation on isolated vertices, i.e, if we want to remove a vertex with  incident edges, then we perform $\Dele(G)$ operations first before proceeding with $\Delv(G)$. 
\end{definition}
Suppose we perform the operation $\Adde(G)$ on $G$ and obtain $G'$. We denote this as $G' = \Adde(G)$ or equivalently $G = \Dele(G')$.

\begin{definition}\label{def:RelocateEdge}\textbf{(Edge Relocation)} Let $G$ be a graph. Suppose an edge relocates from vertices $a$ and $b$ to vertices $u$ and $v$. We denote this by $\relocatee(G_1)$ where $\relocatee(G_1) = \Adde(G_1) + \Dele(G_1)$ where the edge is not uniquely identified by the notation. 
\end{definition}

\begin{definition}\label{def:MergeVertex}\textbf{(Vertex Merge)} Let $G$ be a graph. Suppose two degree-1 vertices merge and become a degree-2 vertex. Suppose the two degree-1 vertices are $a$ and $b$ and $b$ is connected to $c$. Then, merging $a$ and $b$ can be seen as adding the edge $ac$, followed by deleting the edge $bc$ and finally deleting vertex $b$.  We denote vertex merging (for two degree-1 vertices) by $\mergev(G_1)$ where $\mergev(G_1) = \Adde(G_1) + \Dele(G_1) + \Delv(G_1)$, where the edges and vertices are not are not uniquely identified by the notation.
\end{definition}

\begin{definition}\label{def:SplitVertex}\textbf{(Vertex Split)} Let $G$ be a graph. Suppose a degree-2 vertex is split to create two degree-1 vertices. This is the inverse operation of $\mergev(G_1)$.  We denote vertex splitting (for a degree-2 vertex) by $\splitv(G_1)$ where $\splitv(G_1) = \Addv(G_1) + \Dele(G_1) + \Adde(G_1)$, where the edges and vertices are not are not uniquely identified.
\end{definition}

We use these definitions to discuss edge relocations, vertex merging and splitting in both the original graph space and the line graph space which we denote by $G$ space and $H$ space respectively.

\section{Introducing a pseudo-inverse of a line graph}\label{sec:pseudoinverse}


Suppose $G$ is a graph and $H = L(G)$ its line graph. Let  $\widetilde{H}$ be a perturbed version of $H$ where we only consider small perturbations. 
We want to find a ``close'' line graph $\hat{H}$ where we define the notion of closeness as adding or removing the minimum number of edges from/to $\widetilde{H}$ such that the resulting graph is a line graph. Hence, by finding a close line graph $\hat{H}$, we can find the inverse line graph $L^{-1}(\hat{H})$. We call $L^{-1}(\hat{H})$ a pseudo-inverse line graph of {\small $\widetilde{H}$}, which we denote by $L^{\dagger}(\widetilde{H})$. This is shown in Figure \ref{fig:pseudoinverse1}. 

\begin{definition}\label{def:pseudoinverse}\textbf{(A pseudo-inverse of a line graph)}
    Let $\widetilde{H}$ be a graph (which may not be a line graph).
    We define $\hat{G} := L^{\dagger}(\widetilde{H})$ as a pseudo-inverse of $\widetilde{H}$ when $\hat{G}$ has the property that $\hat{H} := L(\hat{G})$ has the minimum number of edge additions or deletions from $\widetilde{H}$, that is
    \[
   L( \hat{G}) = \argmin_{\hat{H}} \left\vert \left(E(\hat{H})\backslash E(\widetilde{H}) \right) \cup \left(E(\widetilde{H})\backslash E(\hat{H}) \right) \right\vert \, , 
    \]
    where $\cup$ defines the union of edges. 
\end{definition}
Note that $L^{\dagger}(\widetilde{H})$ is not unique. While Definition \ref{def:pseudoinverse} encompasses a broader set of perturbations to $H$, we restrict our attention to single edge additions. 

\begin{definition}\label{def:edegaugmentedH}\textbf{(Edge Augmented  $H$)}
    Let $G$ be a graph and $H = L(G)$ its line graph. Let $\widetilde{H} = H  + e$, $\hat{G} = L^{\dagger}(\widetilde{H})$ and $\hat{H} = L(\hat{G})$. We call this scenario ``edge augmented $H$'' .
\end{definition}
In the experiments we show that our method works for more edge additions as well.

\subsection{The different cases}\label{sec:FourCases}
We consider the specific case where $\widetilde{H} = H + e_1$, that is, the edge augmented $H$ scenario. The graph $\hat{H}$ is obtained from $\widetilde{H}$ by adding or removing edges. This set up gives rise to four cases as shown in Figure \ref{fig:ThreeCases} and stated in Theorem \ref{thm:DifferentCases}. 

\begin{restatable}{theorem}{thmDifferentCases}
\label{thm:DifferentCases} For edge augmented $H$ (Definition \ref{def:edegaugmentedH}) exactly one  of the following statements is true.
\begin{enumerate}[label=Case \Roman*:, leftmargin=*]
    \item $\widetilde{H} \cong \hat{H}$, $\hat{H} \ncong H$, $\hat{G} \ncong G$, $L^{\dagger} = L^{-1}$ and either $\hat{G} = \relocatee(G)$ or $\hat{G} = \mergev(G)$.
    \item  $\hat{H} = \Dele(\widetilde{H})$, $\hat{H} \cong H$ and $\hat{G} \cong G$.
    \item $\hat{H} = \Dele(\widetilde{H})$, $\widetilde{H} \ncong \hat{H}$, $\hat{H} \ncong H$, $\hat{G} \ncong G$, $\hat{H} = \relocatee(H)$ and either $\hat{G} = \relocatee(G)$ or $\hat{G} = \mergev(G) + \splitv(G)$.
    \item $\hat{H} = \Adde(\widetilde{H})$, $\widetilde{H} \ncong \hat{H}$, $\hat{H} \ncong H$ and $\hat{G} \ncong G$.
\end{enumerate}
\end{restatable}
\begin{proof}[Proof Sketch:] We refer to the above scenarios as \textit{Case I: $L^\dagger = L^{-1}$}, \textit{Case II: undo}, \textit{Case III: relocate edge} and \textit{Case IV: second add}. Case I can happen when $\widetilde{H}$ is a line graph, i.e.,  $\hat{H} \cong \widetilde{H}$. 
If $\widetilde{H}$ is not a line graph either edges needs to be added or removed. First suppose edges are removed. If $\hat{H}$ is obtained by removing the same (or congruent) edge that was added to $H$, we have Case II (undo), where we end up where we started with $\hat{G} \cong G$ and $\hat{H} \cong H$. If the removed edge is different (or non-congruent) to the one that was added to $H$, then we have Case III (relocate edge) with $\hat{H} \ncong H$ and $\hat{G} \ncong G$. Finally, if $\hat{H}$ is obtained by adding an edge to $\widetilde{H}$, then $\hat{H}$ has extra 2 edges compared to $H$ making $\hat{H} \ncong H$ and $\hat{G} \ncong G$. 
Note that the pseudo-inverse operation does not add or remove more than 1 edge, because the difference between $H$ and $\widetilde{H}$ is one edge and $L^\dagger$ minimizes edge edits.  

For Cases I and III, we show that $\hat{G}$ can be obtained by doing certain modifications to $G$.  Case I has two scenarios: the special case and the general case. The special case (triangle closing) is stated in Lemma \ref{lemma:SpecialCaseWithEdgeAddition} and results in $\hat{G} = \relocatee(G)$. It is illustrated in Figure \ref{fig:SpecialCase2} . For all other scenarios in Case I the general case (Lemma \ref{lemma:GeneralCaseWithEdgeAddition}) applies, which states that $\hat{G} = \mergev(G)$. This is illustrated in Figure \ref{fig:GeneralCase2}.  For Case III (relocate edge) as stated in Lemma \ref{lemma:EdgeRelocationInHspace}, we show that either $\hat{G} = \relocatee(G)$ or $\hat{G} = \mergev(G) + \splitv(G)$.
\end{proof}

We state Lemmas \ref{lemma:SpecialCaseWithEdgeAddition}, \ref{lemma:GeneralCaseWithEdgeAddition} and \ref{lemma:EdgeRelocationInHspace} using the notation $G_1$, $G_2$ for original graphs  and $H_1$, $H_2$ for their line graphs. These lemmas illustrate relationships between graphs and their line graphs without reference to a pseudo-inverse. For this reason we do not use $\hat{G}$ and $\hat{H}$ in their notation. 

\begin{restatable}{lemma}{lemmaSpecialCaseWithEdgeAddition} \textbf{(Special case: triangle closing)}\label{lemma:SpecialCaseWithEdgeAddition}
    Suppose $G_1$ is a graph and $H_1 = L(G_1)$ is its line graph.  Suppose $H_1$ has a degree-2 vertex labelled $c$ and $a$ and $b$ are its neighbours (see Figure \ref{fig:SpecialCase2}). Let us connect $a$ and $b$ with an edge. Then the resulting graph $H_2$ is a line graph, i.e., there exists $G_2$ such that $H_2 = L(G_2)$ where  $G_2$  is obtained from $G_1$ by relocating an edge, $G_2 = \relocatee(G_1)$.
\end{restatable}

\begin{figure}[!ht]
 \centering
 \includegraphics[width=0.8\linewidth]{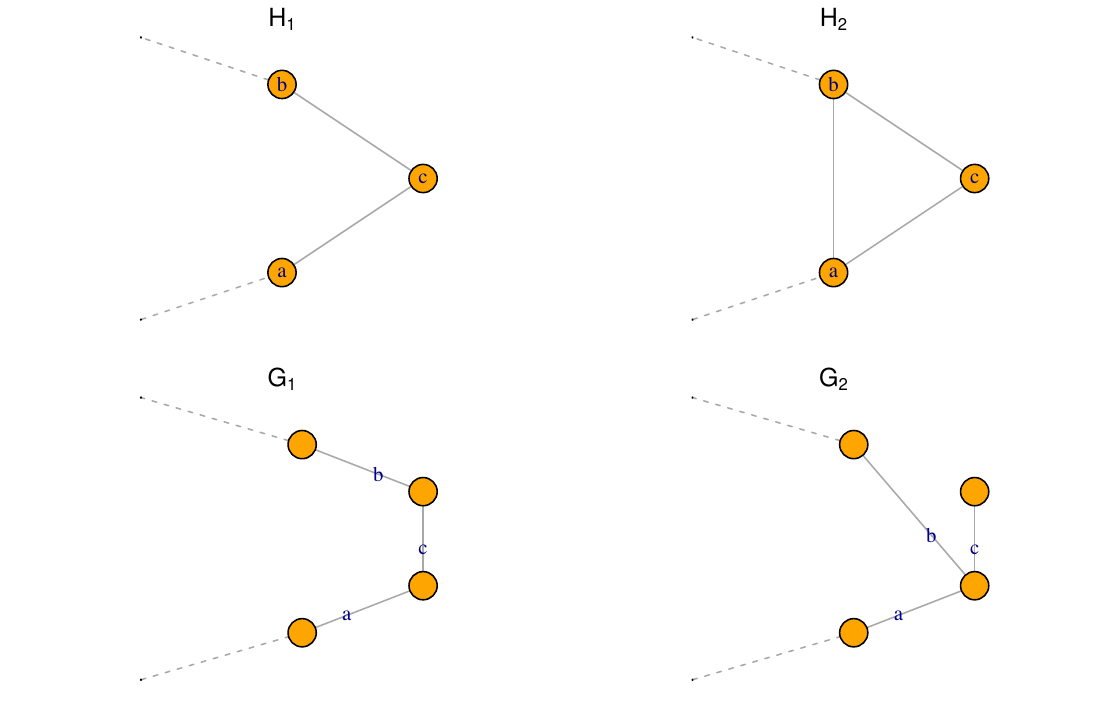}
 \caption{Line graphs $H_1$ and $H_2$, and their inverse line graphs $G_1$ and $G_2$ in the  \textit{triangle closing } scenario. }
 \label{fig:SpecialCase2}
\end{figure}

\begin{restatable}{lemma}{lemmaGeneralCaseWithEdgeAddition}(\textbf{General case})\label{lemma:GeneralCaseWithEdgeAddition}
    Suppose $H_1$ and $H_2$ are line graphs such that $H_2$ is obtained by adding an edge to $H_1$. Let $G_1$ and $G_2$ be the inverse line graphs of $H_1$ and $H_2$ respectively, i.e. $H_1 = L(G_1)$ and $H_2 = L(G_2)$. Then for all cases apart from the triangle closing (Lemma \ref{lemma:SpecialCaseWithEdgeAddition}) $G_2$ is obtained by merging two degree-1 vertices in $G_1$, i.e.,  $G_2 = \mergev(G_1)$. 
\end{restatable}

    \begin{figure}[t]
         \centering
         \centering
        \includegraphics[width=0.8\linewidth, trim=0 5 0 0.5, clip]{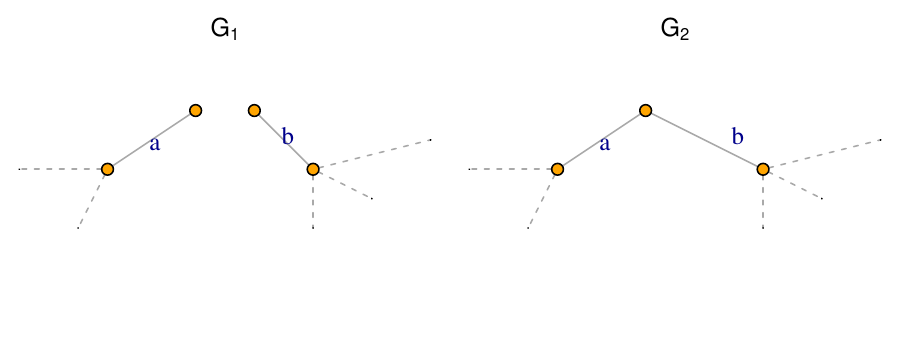}
        \caption{Graph $G_1$ on left with edges $a$ and $b$ not sharing a vertex and graph $G_2$ on the right with edges $a$ and $b$ sharing a  vertex. Possible edges shown in dashed lines. }
        \label{fig:GeneralCase2}
     \end{figure}

\begin{restatable}{lemma}{lemmaEdgeRelocationInHspace}\label{lemma:EdgeRelocationInHspace}
    Let $G_1$, $G_2$ be graphs and $H_1 = L(G_1)$, $H_2 = L(G_2)$ be their line graphs such that $ |V(H_1)| = |V(H_2)|$ and the only difference between $H_1$ and $H_2$ is that a single edge has relocated from $H_1$ to $H_2$. That is, $H_2 = \relocatee(H_1)$. This can only occur in the following scenarios:
    \begin{enumerate}
        \item $G_2 = \relocatee(G_1)$
        \item $G_2 = \mergev(G_1) + \splitv(G_1)$
    \end{enumerate}
\end{restatable}

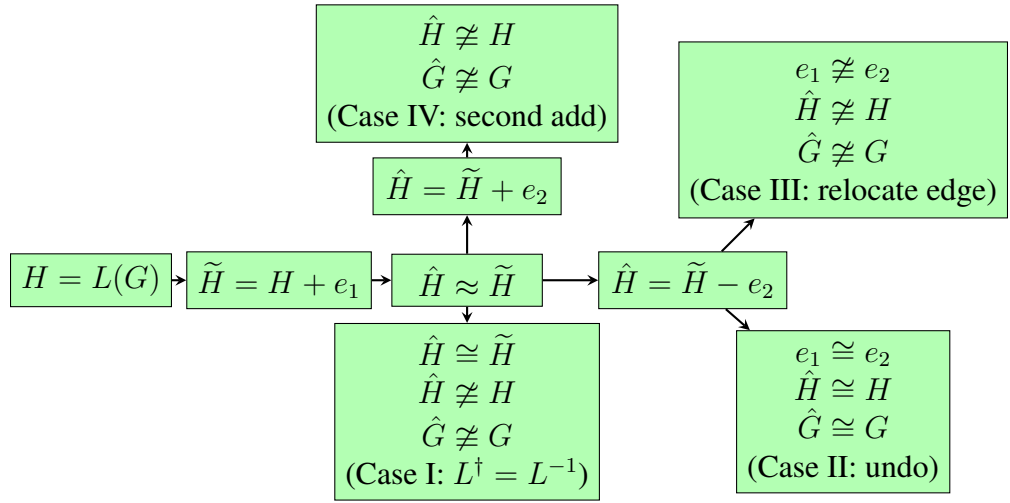
\begin{figure*}[t]
\centering
\begin{tikzpicture}[node distance=1cm]
\node (H) [adjmat2] {$H = L(G)$};
\node (Htilde) [adjmat2, right of=H, xshift=1.5cm] {$\widetilde{H} = H+ e_1$};
\node (Hhat) [adjmat2, right of=Htilde, xshift=1.5cm] {$\hat{H} \approx \widetilde{H}$};
\node (Hhatleft) [adjmat2, below of=Hhat, yshift=-0.75cm, align=center] {$\hat{H} \cong \widetilde{H}$ \\ $\hat{H} \ncong H$ \\ $\hat{G} \ncong G $\\ (Case I: $L^\dagger=L^{-1}$)}; 
\node (Hhatright) [adjmatsmall, right of=Hhat, xshift=2cm] {$\hat{H} =  \widetilde{H} - e_2$};
\node (Hhatadd) [adjmatsmall, above of=Hhat, yshift=0.25cm ] {$\hat{H} =  \widetilde{H} + e_2$};
\node (Hhatright1) [adjmatsmall, below of=Hhatright, xshift=2cm, yshift = -0.75cm, align=center] {$e_1 \cong  e_2 $ \\ $ \hat{H} \cong H$ \\  $\hat{G} \cong G$ \\ (Case II: undo)};
\node (Hhatright2) [adjmatsmall, above of=Hhatright, yshift = 1cm, xshift=2cm, align=center] {$e_1 \ncong  e_2 $ \\ $\hat{H} \ncong H$ \\ $\hat{G} \ncong G$ \\ (Case III: relocate edge)};
\node (Hhat++) [adjmatsmall, above of=Hhatadd, yshift = 0.5cm, align=center] {$\hat{H} \ncong H$ \\ $\hat{G} \ncong G$ \\ (Case IV: second add)};
\draw [arrow] (H) -- (Htilde);
\draw [arrow] (Htilde) -- (Hhat);
\draw [arrow](Hhat) -- (Hhatleft);
\draw [arrow](Hhat) -- (Hhatright);
\draw [arrow](Hhat) -- (Hhatadd);
\draw [arrow](Hhatright) -- (Hhatright1);
\draw [arrow](Hhatright) -- (Hhatright2);
\draw [arrow](Hhatadd) -- (Hhat++);
\end{tikzpicture}
\caption{Different cases}
\label{fig:ThreeCases}
\end{figure*}

\subsection{Spectral radius bounds  between $G$ and $H$ spaces}
The spectral radius of a square matrix $B$, denoted by $\lambda(B)$ is its maximum absolute eigenvalue. For a graph $G$ its spectral radius is the largest eigenvalue of its adjacency matrix $A(G)$. The spectral radius is a global property of the graph and changes to the spectral radius tells us how graph modifications affect its overall connectivity; for example,  the effect on diffusion of information and infection. We use the spectral radius as our graph norm, which we denote by $\lVert G \rVert$ and sometimes by $\lambda(A(G))$. Noting that if $G$ has $n$ vertices and $m$ edges, then $H$ has $m$ vertices, we denote the spectral radii of $G$ and $H$ by $\lVert G \rVert_n$ and $\lVert H \rVert_m$ to distinguish that the graphs are in different spaces.  In more general settings we denote $\lVert \cdot \rVert $ without a subscript.

We explore the relationship of the line graph $H$, its inverse line graph $L^{-1}(H)$ and pseudo-inverse of $L^{\dagger}(\widetilde{H})$ in terms of the spectral radius when $\widetilde{H} = H + e$. Borrowing the definition of bounded linear operators, we show that $L^{-1}$ and $L^{\dagger}$ are bounded without claiming they are linear. 

\begin{definition}\textbf{(Bounded linear operator)} Let $X$ and $Y$ be normed spaces over a scalar field. A linear map $T: X \to Y$ is a bounded linear operator if there is a positive constant $M$ satisfying 
\[
\lVert Tx \rVert_Y \leq M \lVert x \rVert_X \quad \text{for all} \quad x \in X\, . 
\]
 
\end{definition}

To show $L^{-1}$ and $L^{\dagger}$ are bounded, we use a result from  \citet{stevanovic2018spectral} and \citet{SmithGraphs} which categorises graphs with spectral radius $\lambda(A(G)) \leq 2$. 

\begin{theorem}\label{thm:SmithGraphs}\textbf{(Smith Graphs)} \citep{stevanovic2018spectral,SmithGraphs} Connected graphs with $\lambda(A(G)) \leq 2$ are precisely the induced subgraphs shown in Figure \ref{fig:smithgraphs}.
\end{theorem}
Of the Smith graphs the star graph $K_{1,4}$, $W_{n}$, $F_7$, $F_8$ and $F_9$ cannot be line graphs as they contain the line forbidden graph $L_1 = K_{1,3}$. The inverse line graph of cycles $C_n$ are cycles. 

\begin{lemma}\label{lemma:smith}
    Line graphs $H$ that are induced subgraphs of Smith graphs satisfy $\lVert L^{-1}(H) \lVert \leq 2$.
\end{lemma}
\begin{proof}
    The inverse line graphs of paths and cycles are paths and cycles respectively. Suppose $P_n$ is a path of $n$ vertices and $C_n$ is a cycle of $n$ vertices. Then
    \[
    L^{-1}(P_n) = P_{n+1} \quad \text{and} \quad L^{-1}(C_n) = C_{n} \, . 
    \]
    The spectral radius is bounded by the maximum degree of the graph \citep{royle2001algebraic}. Thus, for paths and cycles $H$,  $\lVert H \rVert \leq 2$.
    
    Next we show that line graphs that are induced subgraphs of Smith graphs can only be paths or cycles. 
    Let us go through each of the Smith graphs in Figure \ref{fig:smithgraphs}. Induced subgraphs of cycles $C_n$ can be either be paths or cycles. 
    
    The graph $K_{1, 4}$ cannot be a line graph as it contains the forbidden graph $L_1 = K_{1,3}$ as a subgraph (see Figure \ref{fig:ForbiddenGraphs}). Thus, the induced subgraphs of $K_{1, 4}$ that can be line graphs are paths of length 1 and 2. 

    Similarly, $W_n$, $F_7$, $F_8$ and $F_9$ have $K_{1,3}$ as an induced subgraph and cannot be line graphs. However, the induced subgraphs which do not contain $K_{1,3}$ in these graphs are line graphs. But again, these are path graphs and satisfy $\lVert L^{-1}(P_n) \rVert \leq 2$ as before. 
\end{proof}

\begin{figure}
    \centering
    \includegraphics[width=0.98\linewidth]{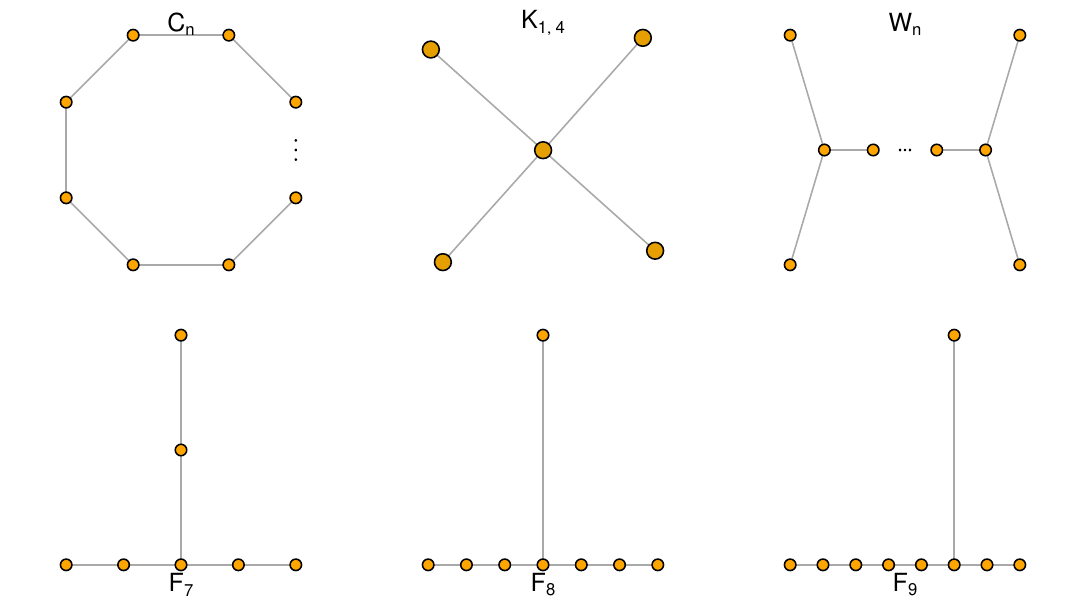}
    \caption{The Smith graphs with $\lambda(A(G)) = 2$.}
    \label{fig:smithgraphs}
\end{figure}

In addition to Smith's graphs we use the following theorem from \cite{beineke2021line}, which gives the relationship between the incidence matrix of a graph $G$ and the adjacency matrix of its line graph $L(G)$.

\begin{theorem}\label{thm:adjacencyandincidence}({\cite{beineke2021line} Theorem 4.4})
    Suppose $G$ is a graph and its incidence matrix is given by $B$. Let $L(G)$ denote the line graph of $G$ and $A(L(G))$ denote the adjacency matrix of $L(G)$. Then
    \begin{equation}\label{eq:incidenceandadjacency}
            A(L(G)) = B'B - 2I \, .
    \end{equation}
\end{theorem}

\begin{proposition}\label{prop:LInverseCts}
    Let $G$ be a graph and $H = L(G)$ its line graph. Then either 
    \[
    \lVert L^{-1}(H) \rVert_n \leq 2  \lVert H \rVert_m   \quad \text{or} \quad  \lVert L^{-1}(H) \rVert_n \leq 2 \, . 
    \]   
 \end{proposition}
\begin{proof}
   The relationship between the adjacency matrix $A(G)$ and the incidence matrix $B$ of a graph $G$ is given by 
    \begin{equation}\label{eq:AdjacencyAndIncidence}
         A(G) = BB' - D
    \end{equation}
    where $D$ is the degree matrix of $G$ defined as the $n \times n$ diagonal matrix with $d_{ii}$ equal to degree of vertex $v_i$.
    Then from equation \eqref{eq:AdjacencyAndIncidence} we get
   \begin{align}\label{eq:GandLambda1BBT}
      \lVert G \rVert_n & =  \lambda_1\left(A(G)\right) = \lambda_1\left( BB' - D \right) \,  \notag \\
      & \leq \lambda_1\left( BB'  \right) - \min(d_{ii}) \,  
      \leq \lambda_1\left( BB'  \right) = \lambda_1(B'B) \, , 
    \end{align}
   where we have used Weyl's inequality and the fact that eigenvalues of $B'B$ are equal to those of $BB'$. 
    From Theorem \ref{thm:adjacencyandincidence}, the adjacency matrix of $H = L(G)$ denoted by $A(H)$ satisfies
    \[
    A(H) = B'B - 2I \, , 
    \]
    where $B$ denotes the incidence matrix of $G$ and $I$ denotes the identity matrix.  Therefore, if $\mu$ is an eigenvalue of $A(H)$,  $\mu + 2$ is an eigenvalue of $B'B$. This gives us 
   \[
   \lambda_1\left( BB'  \right) = \lambda_1\left(A(H) \right) + 2 = \lVert H \rVert_m + 2
   \]
   making
   \begin{equation}\label{eq:NormGAndNormH}
        \lVert G \rVert_n  \leq \lVert H \rVert_m + 2 \, , 
   \end{equation}
   where we have used equation \eqref{eq:GandLambda1BBT}. Only Smith graphs (Theorem \ref{thm:SmithGraphs}) satisfy $\lambda_1(A(G)) \leq 2$. Then for all other line graphs $H$ we have
    \begin{align*}
        \lVert H \rVert_m  & \geq  2 \, , \quad \text{giving us} \\
         \lVert G \rVert_n = \lVert L^{-1}(H) \rVert_n  & \leq 2\lVert H \rVert_m \, , 
    \end{align*}
    where we have used equation \eqref{eq:NormGAndNormH}. For Smith graphs $H$ from Lemma \ref{lemma:smith} we know that $\lVert L^{-1}(H) \rVert_n \leq 2$ giving us the result. 
\end{proof}

\begin{restatable}{lemma}{lemmasmithtwo}\label{lemma:smith2}
    All induced subgraphs $\widetilde{H}$ of Smith graphs that are edge augmented $H$ graphs (Definition \ref{def:edegaugmentedH})  satisfy  $\lVert L^{\dagger}(\widetilde{H}) \rVert \leq 3$. 
\end{restatable}

\begin{restatable}{proposition}{propPseudoinverseCts}\label{prop:PseudoinverseCts}
     In scenario edge augmented $H$ (Definition \ref{def:edegaugmentedH}) for all  graphs $\widetilde{H}$  
    we have
    \[
    \lVert L^{\dagger}(\widetilde{H}) \rVert_n \leq 3 \lVert \widetilde{H} \rVert_m \quad \text{or} \quad  \lVert L^{\dagger}(\widetilde{H}) \rVert_n \leq 3 \, . 
    \]
\end{restatable}
\subsection{Sensitivity to ``small'' perturbations}
In this part we focus on the change of spectral radius when graphs are slightly perturbed. 

\begin{restatable}{theorem}{thmperturbations}\label{thm:perturbations} For edge augmented $H$ for different cases the following statements hold: 
\begin{enumerate}[label=Case \Roman*:, leftmargin=*]
    \item $\frac{ \left \vert  \lVert \hat{G}\rVert_n -  \lVert G \rVert_n \right \vert }{C_G } \leq \frac{\left \vert  \lVert \hat{H} \rVert_m -  \lVert {H} \rVert_m \right \vert }{C_H} \leq 1 \, , $
    \item $\lVert \hat{H} \rVert = \lVert H \rVert$ and $\lVert  \hat{G} \rVert  = \lVert  G \rVert \, ,  $
    \item $\left \vert  \lVert \hat{H} \rVert_m -  \lVert {H} \rVert_m \right \vert \leq  1$\, ,   $\left \vert  \lVert \hat{G}\rVert_n -  \lVert G \rVert_n \right \vert  \leq   2\, , $
    \item  $0 < C \leq \left \vert  \lVert \hat{H} \rVert_m -  \lVert {H} \rVert_m \right \vert \leq  2 \, , $
\end{enumerate}
where $C_G$ depends the graphs in the $G$ space and, $C_H$ and $C$ depends on graphs in the $H$ space. 
\end{restatable}
\begin{proof}[Proof Sketch] For this proof we use results from \cite{li2012bounds}, that state if a graph $F_1$ is perturbed either by adding an edge or removing a vertex (and its adjacent edges), resulting in a graph $F_2$, then the difference in spectral radius is bounded. Consider edge addition. If $F_2 = F_1 + e$ where $e$ connects vertices $i$ and $j$ and  $\bm{x}$ and $\bm{w}$ are the normalized principal eigen vectors of $A(F_2)$ and $A(F_1)$, then they showed that
  \[
   0 <  2w_iw_j \leq \lambda_1(A(F_2)) - \lambda_1(A(F_1)) \leq 2 x_ix_j \leq 1 \, . 
  \]
 If $F_2$ is obtained by  removing vertex $i$ from $F_1$, they showed that 
  \[
  (1 - 2 x_i^2) \lambda_1\left(A(F_{1})\right) \leq \lambda_1\left( A(F_2) \right) \leq \lambda_1\left( A(F_{1}) \right) \, . 
  \]
These two results tell us that edge addition and vertex deletion (and hence addition) cannot change the spectral radius of a graph significantly. 
  
By combining these results from \cite{li2012bounds} we  bound the difference in spectral radius when edges relocate ($F_2 = \relocatee(F_1)$), vertices merge ($F_2 = \mergev(F_1)$) and split ($F_2 = \splitv(F_1)$). 

Consider merging two degree-1 vertices $i$ and $j$, where $j$ is connected to a vertex $k$.  The merging can be thought of as $k$ connecting to $i$ ($\Adde(F_1)$), followed by removing the edge connecting $j$ and $k$ ($\Dele(F_1)$) and finally deleting the vertex $j$ ($\Delv(F_1)$). Splitting can be thought of as the inverse operation and obtained by $\Addv(F_1) + \Adde(F_1) + \Dele(F_1)$. Edge relocation is simply edge addition and edge deletion $\Adde(F_1) + \Dele(F_1)$.  

For certain cases we get inequalities of the form
\begin{equation}\label{eq:upperandlowerbound}
    0 < C_1 \leq \left \vert  \lVert \hat{H} \rVert_m -  \lVert {H} \rVert_m \right \vert \leq C_2 \, , 
\end{equation}
where $C_1$ and $C_2$ are graph dependent constants giving us
\[
\frac{1}{ \left \vert  \lVert \hat{H} \rVert_m -  \lVert {H} \rVert_m \right \vert} \leq \frac{1}{C_1} \, . 
\]
Multiplying with inequalities of the form 
\[
0 \leq \left \vert  \lVert \hat{G} \rVert_m -  \lVert {G} \rVert_m \right \vert \leq C_3 \, , 
\]
we obtain ratios (Cases I and II). When an edge relocates we do not have a strictly positive lower bound $C_1 $ as in equation \eqref{eq:upperandlowerbound}. Thus for these cases we do not have ratios, but we still obtain certain bounds. 
\end{proof}

\section{Estimating a pseudo-inverse line graph}\label{sec:estimation}
To extract a line graph from its noisy version, we use the relationship between the adjacency matrix of a line graph and the incidence matrix of the original graph as described in Theorem \ref{thm:adjacencyandincidence} \citep{beineke2021line}. From equation \eqref{eq:incidenceandadjacency} we have
\[
B'B = A(L(G)) + 2I \, , 
\]
where $A = A(L(G))$ denotes the $m \times m$ adjacency matrix of line graph $L(G)$,  $B$ denotes the $n \times m$ incidence matrix of graph $G$ and $I$ denotes the identity matrix. 
Then the inverse line graph problem is to find an $n \times m$ matrix $B$ such
that
\[
B'B = A + 2I \, , 
\]
for a given $n$. When $A$ is not an adjacency matrix of a line graph, the above equality does not hold.  For matrices $A$ that cannot be
decomposed in this way, we wish to find the minimum number of entries
of $A$ that must be ``flipped'' in order to allow such a result. Thus, given a graph $\widetilde{H}$ that may not be a line graph, we find a line graph $\hat{H}$ by minimizing the number of ``flips'' in $A(\widetilde{H})$.

Related work was conducted by \cite{labbe2021finding} where they focus on the application haplotype phasing. They have a graph called the Clark consistency graph, for which line invertible is useful for finding the set of ancestors that could produce the observed genotypes. From the observed Clark consistency graph they remove the minimum number of edges to obtain a line graph using an integer linear programming formulation. One of the differences between their method and our method is that we consider edge additions as well as edge deletions. Furthermore, their formulation is based on a relationship between line graphs and the graph colouring problem. Ours is based on the relationship of line graphs to incidence matrices (equation \eqref{eq:incidenceandadjacency}). 

Let us use lower case letters to represent the elements of a matrix denoted by the upper case letter, i.e. $A = (a_{ij})_{1 \leq i,j \leq m}$. First we introduce an $m \times m$ matrix $Z = (z_{ij})_{1 \leq i,j \leq m}$ that has entries
\begin{eqnarray*}
a_{ij} = 0 & \rightarrow & z_{ij} = 1  \, , \\
a_{ij} = 1 & \rightarrow & z_{ij} = -1 \, . 
\end{eqnarray*}

This value of $z$ effectively ``flips'' the corresponding $A$
entry. 
We then use the decision matrix $X$ to select flips using the Hadamard product $X \circ Z$ for binary $m \times m$ matrix $X$, where $Y = X \circ Z$ means $y_{ij} = x_{ij}. z_{ij}$.

We then require 
\[
B'B = A + 2I + X \circ Z \, , 
\]
and minimise the number of non-zero elements of $X$. For non-zero elements of $X$, $A$ gets flipped because $a_{ij}$ is replaced with $a_{ij} + x_{ij}z_{ij}$. When $x_{ij} = 0$, $a_{ij}$ remains as it is.  If all entries of $X$ are zero, then $A$ is an adjacency matrix of a line graph.

This gives us problem {\bf P:}

\noindent minimise $\sum_{i,j} x_{ij}$ \\
subject to
\begin{eqnarray}
  B'B & = & A + 2I + X \circ Z \label{Pcons1} \\
  b_{ij} & \in & \{0,1\} \\
  x_{ij} & \in & \{0,1\}
\end{eqnarray}

Note that squared terms $b_{ik} . b_{kj}$ appear in constraint
\eqref{Pcons1} in the product $B'B$, and so {\bf P} cannot be solved
using linear programming. However, we can linearise these elements as
follows.

Entry $(i,j)$ of the product $B'B$ is $\sum_k b_{ik} . b_{kj}$.
Let us define ``product'' variables $p^k_{ij} = b_{ik} . b_{kj}$. Now
$p^k_{ij}$ can only be 0 or 1, and is only 1 when both $b_{ik}$ and
$b_{kj}$ are 1. We can therefore constrain $P$ so that
\begin{eqnarray}
  b_{ik} + b_{kj} & \ge & 2 p^k_{ij}  \label{prodcon1} \\
  b_{ik} + b_{kj} & \le & 1 + p^k_{ij} \label{prodcon2}
\end{eqnarray}
It is clear that $p^k_{ij} = 1$ only when both $b_{ik}$ and $b_{kj}$ are 1. This case satisfies both equations \eqref{prodcon1} and \eqref{prodcon2}. Similarly, $p^k_{ij} = 0$ when either or both of $b_{ik}$ and $b_{kj}$ are 0. This case also satisfies both equations \eqref{prodcon1} and \eqref{prodcon2}.  


We can now formulate problem {\bf ILP:}

\noindent
minimise $\sum_{i,j} x_{ij}$ \\
subject to
\begin{eqnarray}
  \sum_k p^k_{ij} & = & a_{ij} + 2\delta_{ij} + x_{ij} . z_{ij}
  \;\;\; \forall i,j  \label{LPcons1}\\
  b_{ik} + b_{kj} & \ge & 2 p^k_{ij}  \;\;\; \forall i,j,k   \\
  b_{ik} + b_{kj} & \le & 1 + p^k_{ij} \;\;\;  \forall i,j,k   \\
  b_{ij} & \in & \{0,1\} \;\;\;  \forall i,j   \\
  x_{ij} & \in & \{0,1\} \;\;\;  \forall i,j   \\
  p^k_{ij} & \in & \{0,1\} \;\;\;  \forall i,j  
\end{eqnarray}
where $\delta_{ij}$ is the kronecker delta, $\delta_{ij} = 1$ if $i =
j, 0$ otherwise.



The input to the program is the adjacency matrix $A(\widetilde{H})$. As output we get matrices $X\circ Z$ and $B$ where $X\circ Z$ contains the flipped entries of the adjacency matrix $A(\widetilde{H})$ and the matrix $B$ is the incidence matrix of the inverse line graph {$\hat{G}$}. If { $\widetilde{H}$} is a line graph then $X\circ Z = \bm{0}$. We solve problem {\bf ILP} using the Gurobi linear programming solver
(\cite{gurobi}), using default parameters. Our solver is written in C++. Tests were carried out on a machine with 12 64-bit Intel i7-1365U cores and 12Mb of cache (5376 bogomips).

\section{Experiments}\label{sec:experiments}

\subsection{Synthetic examples}
We perform synthetic experiments on 3 graph types: (1) Erdős–Rényi (GNP) (2) preferential attachment (PA) \citep{barabasi1999emergence} and (3) small world graphs (SW) \citep{watts1998collective}.  For each graph type we generate 1000 line graphs $H$ of which 500 graphs are obtained by randomly adding 1 edge to $H$, i.e.,  $\tilde{H} = H + e$ and the other 500 graphs are obtained by randomly adding 5 edges to $H$, i.e.,  $\tilde{H} = H + 5e$. Table \ref{tab:freq-cases} gives the results of these experiments. For all 3 graph types, when 5 edges are added the proportion of graphs satisfying $\hat{H} \cong \widetilde{H}$ reduces drastically, as expected. Furthermore, deleting edges is more prevalent than adding edges. 
 The time taken for different experiments is shown in Figure \ref{fig:timeTaken}.

\begin{table}[t]
    \centering
    \begin{tabular}{c|ccc}
        \toprule
        & $\hat{H} \cong  \widetilde{H}$ & $\Dele(\widetilde{H})$ & $\Adde(\widetilde{H})$\\
        \midrule
        GNP-1  & 0.21 &  0.77 & 0.02\\
        GNP-5 & 0 &  1 & 0.04 \\
        PA-1 & 0.60 &  0.30 & 0.09\\
        PA-5 & 0.06 & 0.85 & 0.23\\
        SW-1 & 0.33 & 0.67 & 0 \\
        SW-5 & 0.01 & 0.99 & 0\\
        \bottomrule
    \end{tabular}
    \caption{The frequency of different edits to recover a line graph from Erd\H{o}s-Reny\'{i} (GNP), preferential attachment (PA) and small world (SW) graphs. $\widetilde{H} = H + e$ is denoted by XX-1 and $\widetilde{H} = H + 5e$ is denoted by XX-5. } 
    \label{tab:freq-cases}
\end{table}

\begin{figure}
    \centering
    \includegraphics[width=0.9\linewidth]{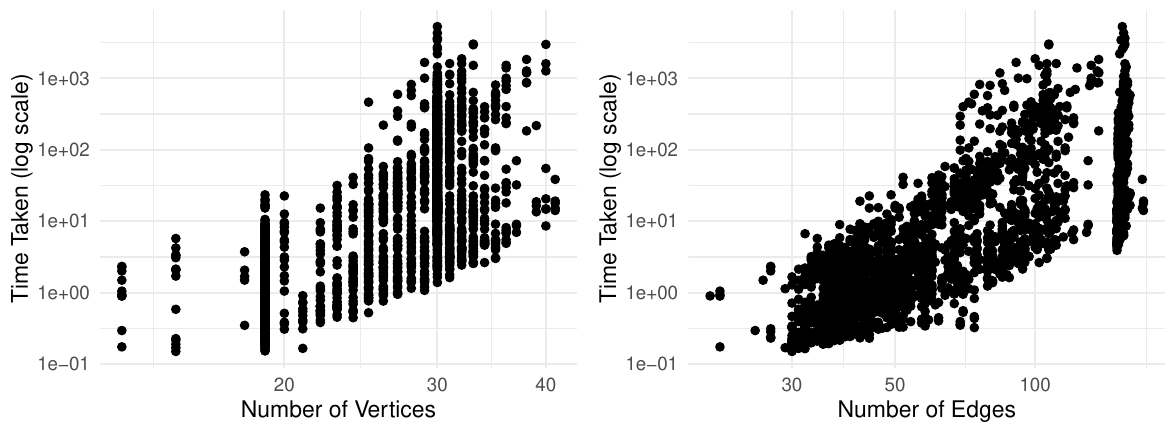}
    \caption{Time taken to find a pseudo-inverse line graph against the number of vertices  and the number of edges. \cheng{don't forget the units (seconds?), I suggest in the y axis label.} }
    \label{fig:timeTaken}
\end{figure}

\subsection{Estimating haplotype population size}
Given a genotype matrix $A$ we construct the Clark Consistency graph $\widetilde{H}$ and find a pseudo-inverse $\hat{G}$ (details in Appendix \ref{sec:AppExperiments}). The number of nodes in $|V(\hat{G})|$ is an  estimate for the size of the ancestor population. We use two 100 sample genotype datasets -- one with $10$ genotypes and the other with $15$ -- from the dataset provided by \cite{ferdosi2014hsphase} to test our method for population estimation. As edge additions are not valid in haplotype phasing, we remove the resulting psuedo-inverses that had edge additions, giving us 29 pseudo-inverses for the 10 genotoype datasets and 79 pseudo-inverses for the 15 genotype datasets. We validate our results using upper and lower bounds estimates discussed in \cite{ferdosi2013effect}. The lower bound estimate is the size of the paternal haplotype population, which indeed is a subset. The upper bound estimate is the size of the haplotype population when both parents are taken into account. This is an upper bound because some parents' alleles are not present in their offspring. 

Figure \ref{fig:haplotypeExample} shows the results with the red curves showing the upper bounds, the blue curves showing the lower bounds and the green curves showing the estimate from the pseudo-inverse method. We see the population estimate from the pseudo-inverse is within the upper and lower bounds.

\begin{figure}[t]
    \centering
    \includegraphics[width=0.98\linewidth]{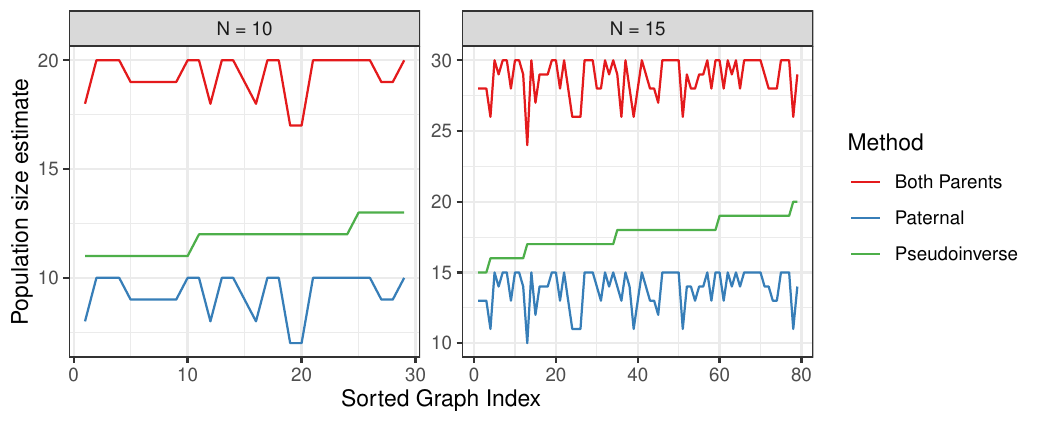}
    \caption{The upper and lower bounds of the parent population size  using methods in \cite{ferdosi2013effect} along with the pseudoinverse estimate of the population size. } 
    \label{fig:haplotypeExample}
\end{figure}
\section{Conclusion}\label{sec:conclusion}

We present a pseudo-inverse of a line graph extending the inverse line graph operation to non-line graphs. Limiting our attention to graphs that are obtained by adding an edge to a line graph, we explore the properties of such a pseudo-inverse.  Using the spectral radius as the graph norm we obtain bounds for the norm of such a pseudo-inverse and show that single edge additions in the line graph space result in small changes to the norm of pseudo-inverses. Furthermore, we propose an integer linear program that finds such a pseudo-inverse minimizing edge additions and deletions. We test our program by estimating the parent population size in genotype data and validate our results.

\bibliographystyle{agsm}
\bibliography{references}

\clearpage
\appendix



\section{Proof of Theorem \ref{thm:DifferentCases}} 

The proof structure of  Theorem \ref{thm:DifferentCases} is given in Figure \ref{fig:Thm3ProofTree}. Of Cases, I, II and III, Case III has many subparts and these are explored in Sections starting with Case III(a), Case III(b) and Case III. 

\begin{figure*}[!ht]
    \centering
\resizebox{\linewidth}{!}{%
\begin{tikzpicture}[node distance=3cm]
\Tree [.{Theorem \ref{thm:DifferentCases}} 
         [.{Case I} 
           [.{Lemma \ref{lemma:case1}} ]
           [.{Lemma \ref{lemma:case1AboutG}} 
             {Lemmas \ref{lemma:SpecialCaseWithEdgeAddition} and \ref{lemma:GeneralCaseWithEdgeAddition}}
           ]
         ]
         [.{Case II} 
            {Lemma \ref{lemma:case2}}
         ]  
         [.{Case III} 
            [.{Lemma \ref{lemma:case3}} ]
            [.{Lemma \ref{lemma:case3second}} 
                {Lemmas in Sections Case III(a), Case III(b) and Case III}
            ]
         ]  
         [.{Case IV}  
            {Lemma \ref{lemma:case4}}
         ]  
      ]  
\end{tikzpicture}   
}
\caption{Proof structure diagram for Theorem \ref{thm:DifferentCases}}
\label{fig:Thm3ProofTree}
\end{figure*}
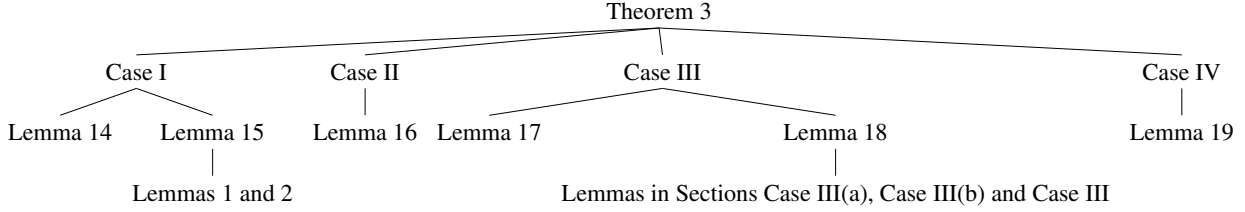


\subsection{Case I: edge addition in $H$ space}\label{sec:EdgeAddInH}

First we discuss Case I, which comprises Lemmas \ref{lemma:SpecialCaseWithEdgeAddition} and \ref{lemma:GeneralCaseWithEdgeAddition}.

\lemmaSpecialCaseWithEdgeAddition*
\begin{proof}
Figure \ref{fig:SpecialCase} shows snippets of the graphs $H_1$, $H_2$,  $G_1$ and $G_2$. The vertices $a$ and $b$ in $H_1$ or $H_2$ can be connected to other vertices, but we are not concerned about those edges. These other possible edges are shown in dashed lines.  

Graph $G_1$, which is the inverse line graph of $H_1$ has edges $a$, $b$, both connected to $c$ but not connected to each other. As vertices $a$ and $b$ are connected in $H_2$,  in $G_2$ edges $a$ and $b$ need to be connected, i.e., they need to share a vertex. This can happen only when edge $b$ detaches itself from the shared vertex with edge $c$ and attaches to the other vertex of edge $c$, which is shared with edge $a$. This is illustrated in Figure \ref{fig:SpecialCase}. This arrangement makes $H_2 = L(G_2)$ with $G_2 = \relocatee(G_1)$.

     
     \begin{figure}[!ht]
         \centering
         \includegraphics[width=0.8\linewidth]{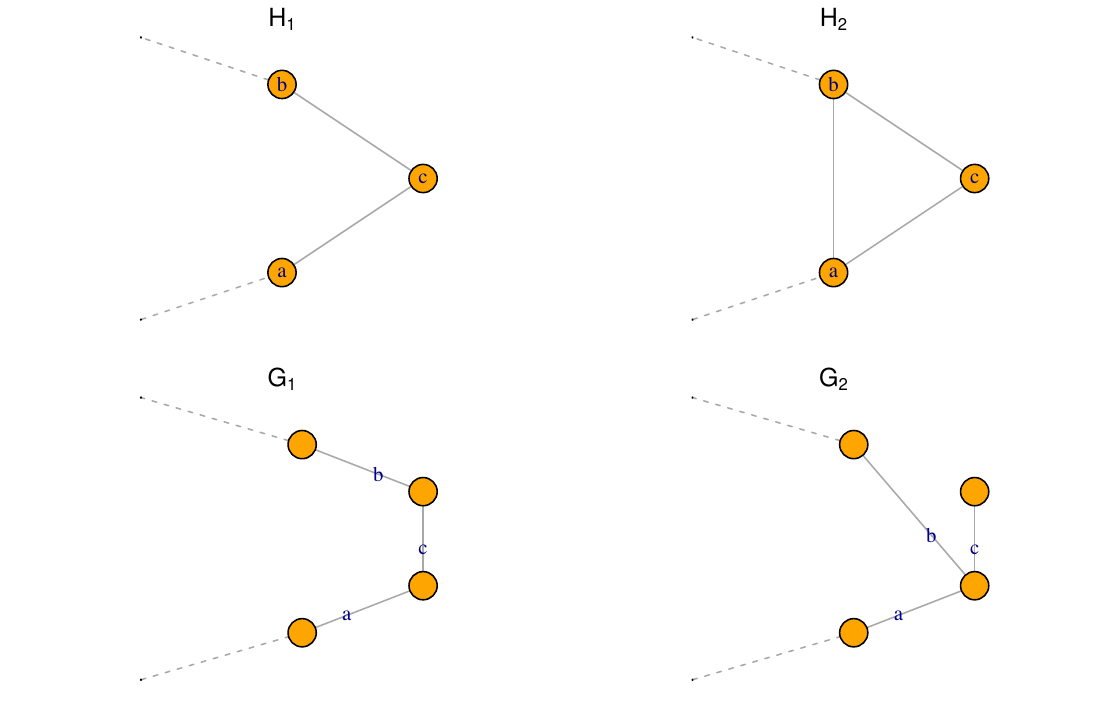}
         \caption{Line graphs $H_1$ and $H_2$, and their inverse line graphs $G_1$ and $G_2$ in the  \textit{triangle closing } scenario. }
         \label{fig:SpecialCase}
     \end{figure}

\end{proof}

\lemmaGeneralCaseWithEdgeAddition*
\begin{proof}
    
    An edge connects two vertices.  Suppose the additional edge in line graph $H_2$ connects vertices $a$ and $b$. These vertices are not connected in $H_1$. Vertices  $a$ and $b$  in $H_1$ and $H_2$ correspond to edges in graphs $G_1$ and $G_2$. In $G_1$ the edges  $a$ and $b$ do not share a vertex as $a$ and $b$ are not connected in $H_1$, but in $G_2$ the edges $a$ and $b$ share a vertex. Apart from the triangle closing (Lemma \ref{lemma:SpecialCaseWithEdgeAddition}) we argue that this can only happen in the following way. 
    
     Suppose in $G_1$ the edges $a$ and $b$ each have a degree-1 vertex. Then these two degree-1 vertices in $G_1$ can merge and become one vertex in $G_2$ making $a$ and $b$ connected in $H_2$. This is shown in Figure \ref{fig:GeneralCase}. 
        \begin{figure}[t]
         \centering
         \centering
        \includegraphics[width=0.8\linewidth, trim=0 5 0 0.5, clip]{Figures/Two_degree_1_vertices_merge_before_and_after.pdf}
        \caption{Graph $G_1$ on left with edges $a$ and $b$ not sharing a vertex and graph $G_2$ on the right with edges $a$ and $b$ sharing a  vertex. Possible edges shown in dashed lines. }
        \label{fig:GeneralCase}
     \end{figure}         
    
   In the current scenario, the two merging vertices in $G_1$ have degree 1. In the triangle closing case, the two vertices that are incident to edge $c$ in $G_1$  had degree 2 and had a common edge $c$.   Suppose there is another scenario in addition to these two scenarios where $a$ and $b$ are connected in $H_2$ but not in $H_1$. Such a scenario needs to consider at least one of the merging vertices in $G_1$ having degree 2 or higher with no common edges with the other merging vertex.  A simple example is shown in Figure \ref{fig:ImpossibleScenario}.  However, if the vertices merge as shown in  Figure \ref{fig:ImpossibleScenario}, we see that not only $a$ and $b$ share a vertex (are connected in $H_2$), but $a$ and $c$ share the same vertex as well. This results in 2 edges being added to $H_2$ ($ab$ and $ac$) compared to $H_1$, which is a contradiction. Therefore, if $H_2$ has only 1 extra edge compared to $H_1$, and it is not the triangle closing (Lemma \ref{lemma:SpecialCaseWithEdgeAddition}), then it is by joining two degree-1 vertices in $G_1$ to obtain $G_2$.

   \begin{figure}[!ht]
         \centering
         \centering
        \includegraphics[width=0.8\linewidth]{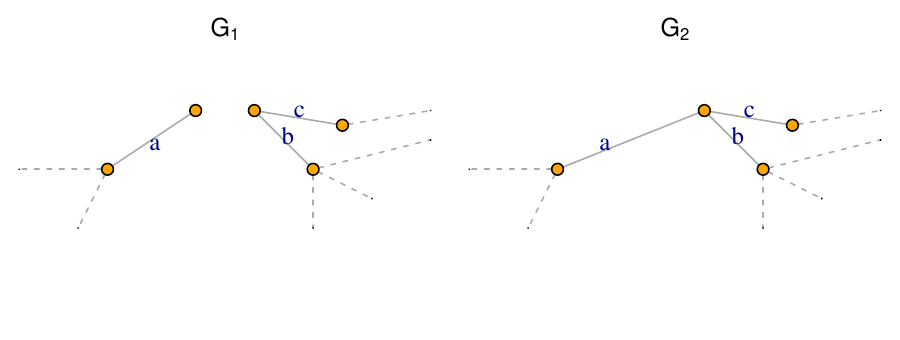}
        \hfill
        \caption{Graph $G_1$ on left with edges $a$ and $b$ not sharing a vertex and graph $G_2$ on the right with edges $a$ and $b$ sharing a  vertex. Possible edges shown in dashed lines. }
        \label{fig:ImpossibleScenario}
     \end{figure}  
\end{proof}

\subsection{Case III (a): edge addition and deletion in $G$ space}\label{sec:EdgeAddInG}

\begin{lemma}\label{lemma:EdgeAddition} Let $G_1$, $G_2$ be graphs such that $G_2 = \Adde(G_1)$, i.e., an edge is added to $G_1$ to form $G_2$. Suppose the edge is added to vertices $u$ and $v$ in $G_1$. Let $H_1 = L(G_1)$ and $H_2 = L(G_2)$ be their line graphs. Then  
\begin{align}
    |V(H_2)| & = V(H_1)| + 1   \quad   \text{and} \\
    |E(H_2)| & =  |E(H_1)| + \deg_{G_1}u + \deg_{G_1}v \, , 
\end{align}
where $\deg_{G_1}u $ and $ \deg_{G_1}v$ refer to the degrees of vertices $u$ and $v$ in $G_1$.
\end{lemma}
\begin{proof}
    As edges of a graph are mapped to vertices to form its line graph a new edge present in $G_2$ increases the number of vertices in $H_2$ by 1 compared to $H_1$. 
    
    Let us call the new edge in $G_2$ as $e_1$.  As $e_1$ connects vertices $u$ and $v$, $e_1$ is connected to all edges incident with $u$ as well as connected to all edges incident with $v$. The edges incident to $u$ in $G_1$ form $\deg_{G_1} u$ number of vertices in $H_1$. As $e_1$ forms a new vertex in $H_2$, this vertex is now connected to all the $\deg_{G_1} u$ vertices in $H_2$. This increases the number of edges in $H_2$ compared to $H_1$ by $\deg_{G_1}u$. Similarly when we consider vertex $v$ we get
    \[
     |E(H_2)|  =  |E(H_1)| + \deg_{G_1}u + \deg_{G_1}v \, .
    \]
\end{proof}

\begin{lemma}\label{lemma:EdgeDeletion} Let $G_1$, $G_2$ be graphs such that $G_2 = \Dele(G_1)$, i.e., an edge is deleted from $G_1$ to form $G_2$. Suppose the edge is deleted from vertices $u$ and $v$ in $G_1$. Let $H_1 = L(G_1)$ and $H_2 = L(G_2)$ be their line graphs. Then  
\begin{align}
    |V(H_2)| & = V(H_1)| - 1   \quad   \text{and} \\
    |E(H_2)| & =  |E(H_1)| - \deg_{G_1}u - \deg_{G_1}v + 2 \, , 
\end{align}
where $\deg_{G_1}u $ and $ \deg_{G_1}v$ refer to the degrees of vertices $u$ and $v$ in $G_1$.
\end{lemma}
\begin{proof}
    This is the same as adding an edge to $G_2$ to obtain $G_1$. Then from Lemma \ref{lemma:EdgeAddition} 
    \begin{align}
    |V(H_1)| & = V(H_2)| + 1   \quad   \text{and} \\
    |E(H_1)| & =  |E(H_2)| + \deg_{G_2}u + \deg_{G_2}v \, .
\end{align}
As the edge is removed from $G_1$ to obtain $G_2$, $\deg_{G_2}u = \deg_{G_1}u - 1$ and  $\deg_{G_2}v = \deg_{G_1}v - 1$ giving the result. 
\end{proof}

\begin{lemma}\label{lemma:EdgeAddAndDelete}Let $G_1$, $G_2$ be graphs such that $G_2 = \relocatee(G_1)$. That is, an edge has relocated in $G_1$ to form $G_2$. Suppose the new edge in $G_2$ connects vertices $u$ and $v$ and the deleted edge in $G_1$ connected vertices $a$ and $b$. Let $H_1 = L(G_1)$ and $H_2 = L(G_2)$ be their line graphs. Then  
\begin{align}
    |V(H_2)| & = V(H_1)|    \quad   \text{and} \\
    |E(H_2)| & =  |E(H_1)| + \deg_{G_1}u + \deg_{G_1}v  \\
    & \quad - \deg_{G_1}a - \deg_{G_1}b + 2  \, .
\end{align}
If $u = a$, then it results in $H_2$ having $\deg_{G_1}v$ new edges and deleting  $\deg_{G_1}b - 1$ edges.  
\end{lemma}
\begin{proof}
    We get the two equations by combining lemmas \ref{lemma:EdgeAddition} and \ref{lemma:EdgeDeletion}.  If $u = a$, then an edge $e$ has switched from vertex $b$ to $v$. In this case, it adds $\deg_{G_1} v$ edges from Lemma \ref{lemma:EdgeAddition}. As $e$ is no longer incident to vertex $b$, it is no longer connected to the other $\deg_{G_1} b - 1$ edges incident to $b$. This it removes $\deg_{G_1} b - 1$ edges from $H_1$ to obtain $H_2$.
\end{proof}

\begin{lemma}\label{lemma:EdgeSwapGLem1} Let $G_1$, $G_2$ be graphs such that and edge $e$ connecting vertices $a$ and $b$ has switched from vertex $b$ to $v$ to form $G_2$. Suppose the resulting line graphs $H_1 = L(G_1)$ and $H_2 = L(G_2)$ differ by an edge relocation, i.e., $H_2 = \relocatee(H_1)$. Then, $\deg_{G_1} v = 1$ and $\deg_{G_1} b = 2$.
\end{lemma}
\begin{proof}
    From Lemma \ref{lemma:EdgeAddAndDelete} an edge relocation with one vertex change from vertex $b$ to $v$ adds $\deg_{G_1} v$ edges and deletes $\deg_{G_1} b - 1$ edges from $H_1$ to $H_2$.  As only one edge is added $\deg_{G_1} v = 1$. Similarly, as only one edge is deleted we have  $\deg_{G_1} b - 1 = 1$ giving the result. 
\end{proof}

\SK{This following Lemma needs to be scrutinized and updated.}
\begin{lemma}\label{lemma:EdgeSwapGLem2} Let $G_1$, $G_2$ be  graphs such that an edge $e$ connecting vertices $a$ and $b$ in $G_1$ has relocated to vertices $u$ and $v$ to form $G_2$ where $u$ and $v$ are different vertices from $a$ and $b$. Suppose the resulting line graphs $H_1 = L(G_1)$ and $H_2 = L(G_2)$ differ by an edge relocation, i.e., $H_2 = \Adde(H_1) + \Dele(H_1)$. Then, vertices $a$ and $b$ in $G_1$ have degrees 2 and 3, and vertices $u$ and $v$ in $G_1$ have degrees 2 and 1. Furthermore, $u$ and $v$ are neighbours of $a$.
\end{lemma}
\begin{proof}
This is a special case shown in Figure \ref{fig:CaseIIIC}. Edge 4 detaches from vertices $a$ and $b$ in $G_1$ and attaches to $u$ and $v$ in $G_2$ resulting in edge 4-5  getting deleted in $H_1$ and edge 1-4 getting added in $H_2$. Even though  edge 4 has relocated both vertices from $G_1$ to $G_2$, it is still incident to edges 2 and 3 because vertex $a$ is a neighbour of $u$ and $v$. The labels $u$ and $v$ can swap and similarly $a$ and $b$ can swap in the following discussion. 

From Lemma \ref{lemma:EdgeAddition} we know that the number of edges from $H_1$ to $H_2$ increase by $\deg_{G_1} u + \deg_{G_1} v$ and the number of edges decrease by $\deg_{G_1} a + \deg_{G_1} b -2 $. As there is an edge relocation from $H_1$ to $H_2$ we have $\deg_{G_1} u + \deg_{G_1} v = 1$ implying that $\deg_{G_1} u$ and $\deg_{G_1} v$ can only take the values 1 and 0 as degrees are not negative. That is, either $u$ or $v$ is an isolated vertex in $G_1$, which is akin to a vertex addition scenario.  As we are considering edge addition and deletion without vertex addition or deletion, we disregard this option.  

What if edges are shared between vertices $u$, $a$ and $b$? Vertices $u$ and $a$ can share an edge, and similarly $u$ and $b$ can share an edge. In this instance, in addition to the relocating edge, two edges are counted in $\deg_{G_1} u + \deg_{G_1} v$ and $\deg_{G_1} a + \deg_{G_1} b -2 $ making $\deg_{G_1} u + \deg_{G_1} v = 3$. This can only happen when wlog $\deg_{G_1} u = 2$ and $\deg_{G_1} v = 1$. Similarly counting the edges $au$ and $bu$ we get  $\deg_{G_1} a + \deg_{G_1} b -2 = 3 $ making $\deg_{G_1} a$ and $\deg_{G_1} b$ either 2 and 3 or 1 and 4 (permutations excepted). However, $\deg_{G_1} a$ and $\deg_{G_1} b$ cannot take the values 1 and 4 because this means after deleting edge $e$ from vertices $a$ and $b$, more edges would be deleted from $H_1$ to $H_2$ . Thus, vertices $a$ and $b$ have degrees 2 and 3. From Figure \ref{fig:CaseIIIC} we see that if either vertex $a$ or $b$ in $G_1$ are incident to additional edges then relocating edge 4 from $a$ and $b$ to $u$ and $v$ results in more edge deletions. 

  \begin{figure}
            \centering
            \includegraphics[width=0.8\linewidth]{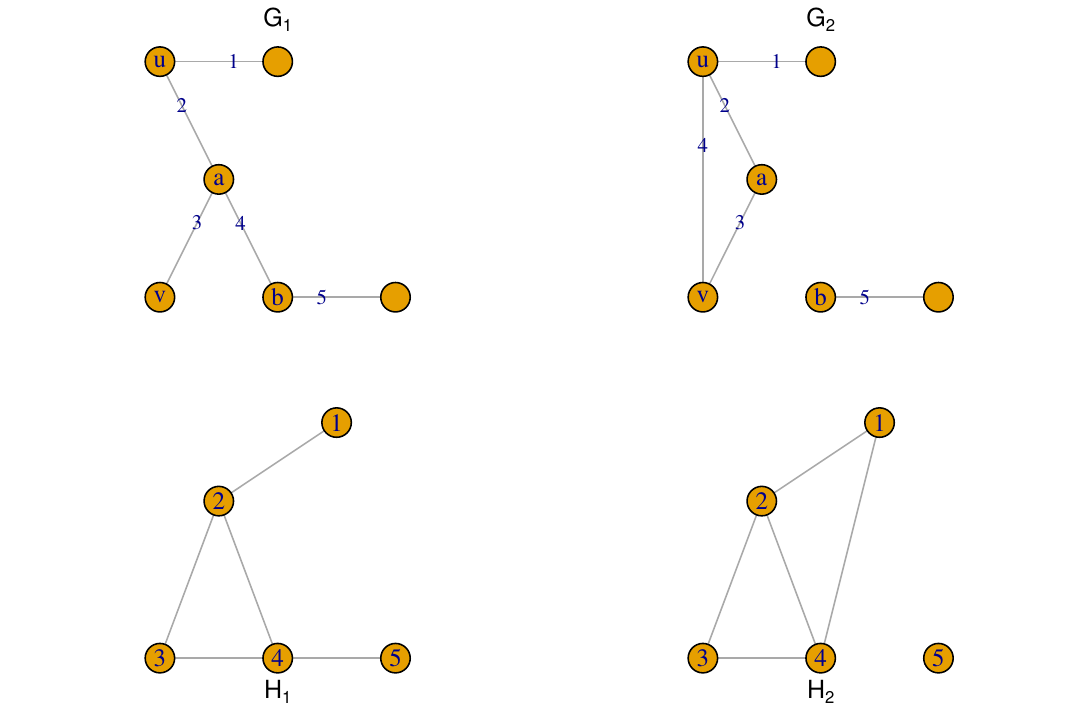}
            \caption{Case in Lemma \ref{lemma:EdgeSwapGLem2}}
            \label{fig:CaseIIIC}
\end{figure}
\end{proof}

\subsection{Case III (b): vertex merging and splitting in $G$ space}\label{sec:SecIIIb}
\begin{lemma}\label{lemma:VertexMerge}
Let $G_1$, $G_2$ be graphs such that $G_2$ is obtained by merging two degree-1 vertices in $G_1$, i.e.  $G_2= \mergev(G_1)$.  Let $H_1 = L(G_1)$ and $H_2 = L(G_2)$ be their line graphs. Then $H_2$ differs from $H_1$ by an edge addition, i.e.,  $H_2  = \Adde(H_1)$ with $|V(H_2)|  = V(H_1)|$ and $ |E(H_2)|  =  |E(H_1)| + 1$.
\end{lemma}
\begin{proof}
Lemma \ref{lemma:GeneralCaseWithEdgeAddition} shows that an edge addition in the $H$ can be accounted for by vertex merging in the $G$ space. When two degree-1 vertices merge, the respective edges share a vertex making the corresponding vertices in the line graph space connected.
\end{proof}

\begin{lemma}\label{lemma:VertexSplit}
Let $G_1$, $G_2$ be graphs such that $G_2$ is obtained by splitting a degree-2 vertex in $G_1$, i.e.  $G_2= \splitv(G_1)$.  Let $H_1 = L(G_1)$ and $H_2 = L(G_2)$ be their line graphs. Then $H_2$ differs from $H_1$ by an edge deletion, i.e.,  $H_2  = \Dele(H_1)$ with $|V(H_2)|  = V(H_1)|$ and $ |E(H_2)|  =  |E(H_1)| - 1$.
\end{lemma}
\begin{proof}
 This is the reverse of Lemma \ref{lemma:VertexMerge}. 
\end{proof}

\begin{lemma}\label{lemma:MergeAndSplit}
Let $G_1$, $G_2$ be graphs such that $G_2$ is obtained by merging two degree-1 vertices in $G_1$ and splitting a degree-2 vertex in $G_1$ to make 2 degree-1 vertices in $G_2$. That is,  $G_2= \mergev(G_1) + \splitv(G_1)$.  Let $H_1 = L(G_1)$ and $H_2 = L(G_2)$ be their line graphs. Then $H_2$ differs from $H_1$ by an edge relocation, i.e.,  $H_2  = \relocatee(H_1)$ with $|V(H_2)|  = V(H_1)|$ and $ |E(H_2)|  =  |E(H_1)|$.
\end{lemma}
\begin{proof}
    Let $G_{12}$ denote the in-between graph from $G_1$ to $G_2$ where $G_{12} = \mergev(G_1)$ and $G_2 = \splitv(G_{12})$ and let $H_{12} = L(G_{12})$. Then applying Lemma \ref{lemma:VertexMerge} to $G_{1}$ and $G_{12}$  and Lemma \ref{lemma:VertexSplit} to graphs $G_{12}$ and $G_2$ we get the result. 
 \end{proof}

\subsection{Case III: edge relocation in $H$ space} \label{sec:CaseIII}
In this section we combine Case III (a) and Case III (b).

\lemmaEdgeRelocationInHspace*
\begin{proof}
    As edges in $G_1$ and $G_2$ are mapped to vertices in $H_1$ and $H_2$ and as $ |V(H_1)| = |V(H_2)|$ we know that $ |E(G_1)| = |E(G_2)|$. Thus, the change from $G_1$ to $G_2$ does not consider only an edge addition. Nor can it consider only an edge deletion. Rather, it can consider edge relocations which is $\Adde(G_1) + \Dele(G_1)$. Lemmas \ref{lemma:EdgeAddAndDelete}, \ref{lemma:EdgeSwapGLem1} and \ref{lemma:EdgeSwapGLem2} show that edge relocation in $G$ space result in an edge relocation in $H$ space. 
    However, it is not the case that multiple edge relocations in $G$ space can cause a single edge relocation in $H$ space because this implies that edge relocations apart from one had no effect, i.e., they cancelled out each other.     

    Similarly, Lemma \ref{lemma:MergeAndSplit} shows that vertex merging and splitting in $G$ space result in edge relocation in $H$ space. If multiple sets of vertices merged and split in $G$ space but still resulted in a single edge relocation in $H$ space, this means that apart from one split and merge the others cancelled out each other. Thus, only a single vertex merge and a single split can result in an edge relocation in $H$ space. 

    Furthermore, it cannot be the case that a vertex merge and an edge relocation can happen in $G$ space, because it would reduce the number of vertices in the $H$ space. As the number of vertices in $H_1$ is the same as that of $H_2$ vertex merging need to balanced with splitting. Similarly, other combinations of \textit{Primary Operations} (Definition \ref{def:AddingDeletingEdgesVertices}) in $G$ space would result in a different number of vertices in $H$ space.     
\end{proof}

\subsection{Assembling different cases to prove Theorem \ref{thm:DifferentCases}}

\begin{lemma}(\textbf{Case I})\label{lemma:case1} 
For edge augmented $H$ (Definition \ref{def:edegaugmentedH}) if  $\widetilde{H} \cong \hat{H}$ then $L^\dagger = L^{-1}$.
\end{lemma}
\begin{proof}
    If $\widetilde{H} \cong \hat{H}$ then $\widetilde{H}$ is a line graph. Thus, $L^{-1}(\hat{H}) = L^{-1}(\widetilde{H})$. As  $L^{-1}(\hat{H}) = L^{\dagger}(\widetilde{H})$ we have $L^\dagger = L^{-1}$ in this instance. 
\end{proof}

\begin{lemma}(\textbf{Case I})\label{lemma:case1AboutG} 
For edge augmented $H$ (Definition \ref{def:edegaugmentedH}) if  $\widetilde{H} \cong \hat{H}$ then $G$ and $\hat{G}$ satisfy either the \textit{Special case} (Lemma \ref{lemma:SpecialCaseWithEdgeAddition})  or the \textit{General case} (Lemma \ref{lemma:GeneralCaseWithEdgeAddition}), i.e., either $\hat{G} = \relocatee(G)$ or $\hat{G} = \mergev(G)$.
\end{lemma}
\begin{proof}
As $L(\hat{G}) = \hat{H} \cong \widetilde{H} = \Adde(H)$ from Lemmas \ref{lemma:SpecialCaseWithEdgeAddition} and \ref{lemma:GeneralCaseWithEdgeAddition} we know that  
either $\hat{G} = \relocatee(G)$ or $\hat{G} = \mergev(G)$.
\end{proof}

\begin{lemma}(\textbf{Case II})\label{lemma:case2} 
For edge augmented $H$ (Definition \ref{def:edegaugmentedH}) suppose $\hat{H} =  \Dele(\widetilde{H}) = \widetilde{H} - e_2$. Then 
\[ e_1 \cong e_2  \iff  \hat{H} \cong H \iff \hat{G} \cong G . 
\]    
\end{lemma}
\begin{proof}
    Graph $\widetilde{H}$ is obtained by adding edge $e_1$ to $H$. If we remove the same edge or an isomorphic edge to obtain $\hat{H}$, then we get back $H$, i.e. $H \cong \hat{H}$, which implies $G \cong \hat{G}$ (Theorem \ref{thm:IsomorphicLineGraphs}). Similarly if  $H \cong \hat{H}$ then $e_1 \cong e_2$ and $\hat{G} \cong G$.
\end{proof}

\begin{lemma}(\textbf{Case III})\label{lemma:case3} 
For edge augmented $H$ (Definition \ref{def:edegaugmentedH}) suppose $ \hat{H} = \Dele(\widetilde{H}) = \widetilde{H} - e_2$. If $e_1 \ncong e_2$, then $\hat{H} \ncong H$ and $ \hat{G} \ncong G$.    
\end{lemma}
\begin{proof}
    This is the contrapositive of Lemma \ref{lemma:case2}.
\end{proof}

\begin{lemma}(\textbf{Case III})\label{lemma:case3second} 
For edge augmented $H$ (Definition \ref{def:edegaugmentedH}) suppose $ \hat{H} = \Dele(\widetilde{H}) = \widetilde{H} - e_2$. If $e_1 \ncong e_2$, then $\hat{H} = \relocatee(H)$ and either $\hat{G} = \relocatee(G)$ or $\hat{G} = \mergev(G) + \splitv(G)$.
\end{lemma}
\begin{proof}
 As $\widetilde{H} = H + e_1 = \Adde(H)$,  $\hat{H} = \Dele(\widetilde{H}) $ and $e_1\ncong e_2$ we have $\hat{H} = \relocatee(\widetilde{H})$. From Lemma \ref{lemma:EdgeRelocationInHspace} we get the result. 
 \end{proof}

\begin{lemma}(\textbf{Case IV})\label{lemma:case4} 
For edge augmented $H$ (Definition \ref{def:edegaugmentedH}) suppose $\hat{H} = \Adde(\widetilde{H}) = \widetilde{H} + e_2$. Then $\hat{H} \ncong H$, $\hat{H} \ncong \widetilde{H}$ and $\hat{G} \ncong G$.
\end{lemma}
\begin{proof}
    As $\hat{H} = \Adde(\widetilde{H})$, $\hat{H} \ncong \widetilde{H}$. As $\hat{H}$ has two additional edges compared to $H$, $\hat{H} \ncong {H}$. Thus, $\hat{G} \ncong G$ from Theorem \ref{thm:IsomorphicLineGraphs}.
\end{proof}

\thmDifferentCases*
\begin{proof}
    The cases are done separately in Lemmas \ref{lemma:case1}, \ref{lemma:case1AboutG},\ref{lemma:case2}, \ref{lemma:case3}, \ref{lemma:case3second} and \ref{lemma:case4}
\end{proof}

\section{Smith graphs related proofs}
\lemmasmithtwo*
\begin{proof}
   
    As $\widetilde{H}$ is an induced subgraph of a Smith graph that also satisfies the edge augmented $H$ scenario, it is at most one edge away from a line graph $H$. Either $\widetilde{H}$ is a line graph or a line graph can be recovered from $\widetilde{H}$ by adding or deleting one edge. If $\widetilde{H}$ is a line graph then from Lemma \ref{lemma:smith},  $\lVert L^{-1}(\widetilde{H}) \rVert \leq 2$. 
      
      Suppose $\widetilde{H}$ is not a line graph. Let $L^{\dagger}(\widetilde{H}) = L^{-1}(\hat{H})$. If $\hat{H}$ is obtained by removing an edge from $\widetilde{H}$, then $\hat{H}$ is an induced subgraph of a Smith graph and as such from Lemma \ref{lemma:smith} we have $\lVert  L^{-1}(\hat{H}) \rVert \leq 2$. As $L^{\dagger}(\widetilde{H}) = L^{-1}(\hat{H})$ we get the result for this case.  

      If $\hat{H}$ is obtained by adding an edge from $\widetilde{H}$, then $\hat{H}$ is not an induced subgraph of a Smith graph. We consider this scenario by going over each of the Smith graphs. This cannot occur for cycles $C_n$ because all induced subgraphs of cycles are line graphs. Let us consider the other graphs one by one. 

     For $K_{1,4}$, graph $\widetilde{H}$ can only be $K_{1,3}$ and if an edge is added it will produce $\hat{H}$ and $\hat{G} = L^{-1}(\hat{H}) = L^{\dagger}(\widetilde{H})$ as shown in Figure \ref{fig:HhatSmith1}.  As the maximum degree of $\hat{G} = 3$, using the fact that spectral radius is bounded by the maximum degree of a graph we get $ \lVert L^{\dagger}(\widetilde{H}) \rVert  = \lVert \hat{G} \rVert \leq 3$. 

     \begin{figure}[!ht]
         \centering
         \includegraphics[width=0.8\linewidth]{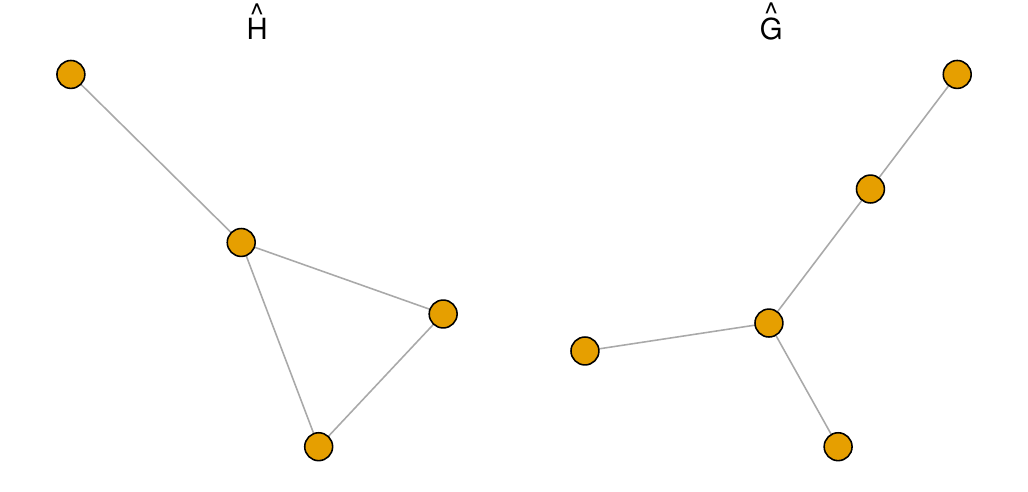}
         \caption{Graphs $\hat{H}$ and $\hat{G} = L^{-1}(\hat{H})$ for $\widetilde{H} = K_{1,3}$}
         \label{fig:HhatSmith1}
     \end{figure}

    Next, let us consider the family of graphs $W_n$ shown in Figure \ref{fig:smithgraphs}. The induced subgraphs that are not line graphs but are edge augmented $H$ graphs contain one copy of $K_{1,3}$. The associated $\widetilde{H}$ and $\hat{H}$ graphs resemble those shown in Figure \ref{fig:HhatSmith2} where more edges can be added to one side of $\widetilde{H}$. 

    \begin{figure}[!ht]
        \centering
        \includegraphics[width=0.8\linewidth]{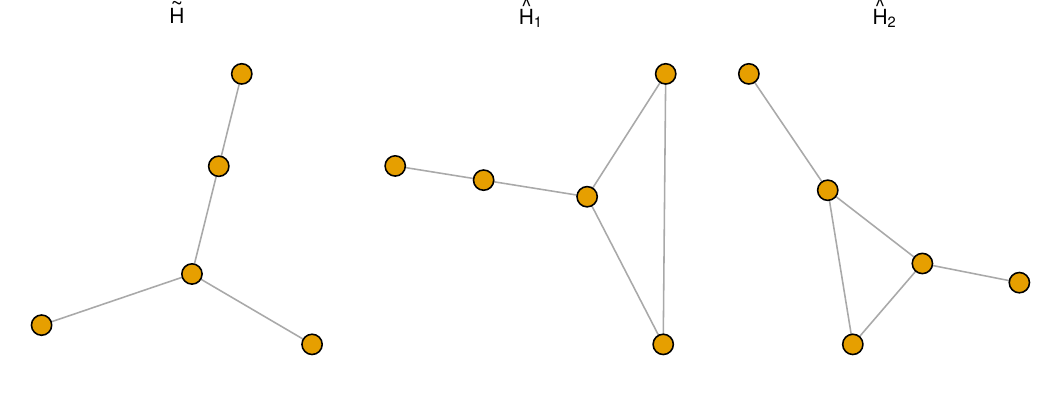}
        \caption{Graphs $\hat{H}_1$ and $\hat{H}_2$  for induced subgraph $\widetilde{H}$ in $W_n$. Note that more edges can be added to $\widetilde{H}$ resulting in the path part of $\hat{H}$ being extended. }
        \label{fig:HhatSmith2}
    \end{figure}

    The inverse line graphs $\hat{G} = L^{\dagger}(\widetilde{H})$ for $\hat{H}_1$ and $\hat{H}_2$ are given in Figure \ref{fig:HhatSmith21}.

    \begin{figure}[!ht]
        \centering
        \includegraphics[width=0.8\linewidth]{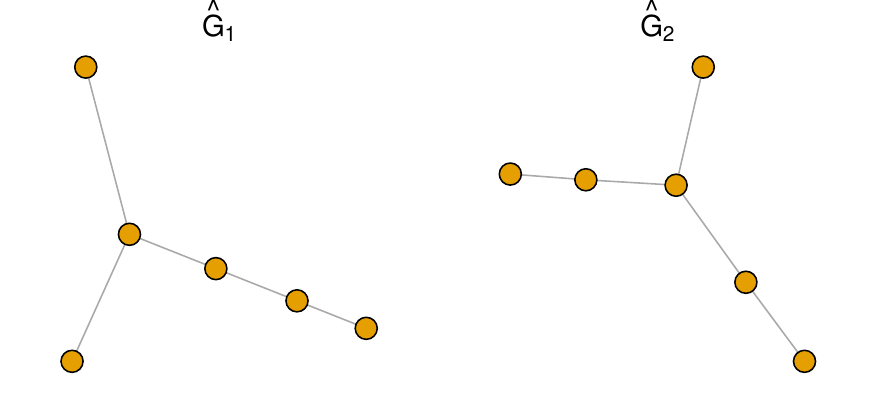}
        \caption{Graphs $\hat{G}_1$ and $\hat{G}_2$ for $\hat{H}_1$ and $\hat{H}_2$ above. }
        \label{fig:HhatSmith21}
    \end{figure}

    Again, as the maximum degree of $\hat{G}_1$ and $\hat{G}_2$ is 3, $\lVert \hat{G} \rVert = \lVert L^{\dagger}(\widetilde{H}) \rVert \leq 3$. 

    Similarly for $F_7$, $F_8$ and $F_9$ the induced subgraphs $\widetilde{H}$ and $\hat{H}$ resemble those in Figure \ref{fig:HhatSmith2}, making  
    $\lVert L^{\dagger}(\widetilde{H}) \rVert \leq 3$.
   
\end{proof}

\propPseudoinverseCts*
\begin{proof}
    As $L^\dagger(\widetilde{H}) = L^{-1}(\hat{H})$ from Proposition \ref{prop:LInverseCts} we have
    \begin{equation}\label{eq:prevresult}
        \lVert L^\dagger(\widetilde{H}) \rVert_n  \leq 2 \lVert \hat{H} \rVert_m  \quad \text{or} \quad  \lVert L^{\dagger}(\widetilde{H}) \rVert_n \leq 2 \, . 
    \end{equation}
    As $\widetilde{H} = H + e$, either $\hat{H} \cong \widetilde{H}$,   $\hat{H} = \widetilde{H} - e$ or $\hat{H} = \widetilde{H} + e$. 
    From eigenvalue interlacing theorems \cite{royle2001algebraic} 
    we know that removing an edge from a graph reduces its largest eigenvalue. Hence for $\hat{H} \cong \widetilde{H}$ or  $\hat{H} = \widetilde{H} - e$ we get 
    \[
    \lVert \hat{H} \rVert_m = \lambda_1(A(\hat{H})) \leq  \lambda_1(A(\widetilde{H})) = \lVert \widetilde{H} \rVert_m  \, . 
    \]
    Combining with equation \eqref{eq:prevresult} we get
    \[
    \lVert L^{\dagger}(\widetilde{H}) \rVert_n \leq 2 \lVert \widetilde{H} \rVert_m \, ,
    \]
    when $\hat{H} \cong \widetilde{H}$ or  $\hat{H} = \widetilde{H} - e$. When $\hat{H} = \widetilde{H} + e$ from \cite{li2012bounds} Lemma 1 and Corollary 1, the spectral radius of $\hat{H}$ satisfies
    \[
    \lVert \widetilde{H}  \rVert_m \leq \lVert \hat{H} \rVert_m \leq  \lVert \widetilde{H}  \rVert_m + 1
    \]
    making
    \begin{align*}
         \lVert L^{\dagger}(\widetilde{H}) \rVert_n & \leq 2 \left( \lVert \widetilde{H}  \rVert_m + 1 \right) \, .
    \end{align*}
    where we have used equation \eqref{eq:prevresult}. For all $\widetilde{H}$ apart from Smith graphs we have $ \lVert \widetilde{H}  \rVert_m \geq 2$ giving us 
    \[
     \lVert L^{\dagger}(\widetilde{H}) \rVert_n  \leq 2\lVert \widetilde{H}  \rVert_m + \lVert \widetilde{H}  \rVert_m = 3 \lVert \widetilde{H}  \rVert_m \, . 
    \]
    For Smith graphs that satisfy edge augmented $H$ criteria from Lemma \ref{lemma:smith2} we have $ \lVert L^{\dagger}(\widetilde{H}) \rVert_n  \leq 3$ giving the result. 
\end{proof}

\section{Proof of Theorem \ref{thm:perturbations}}
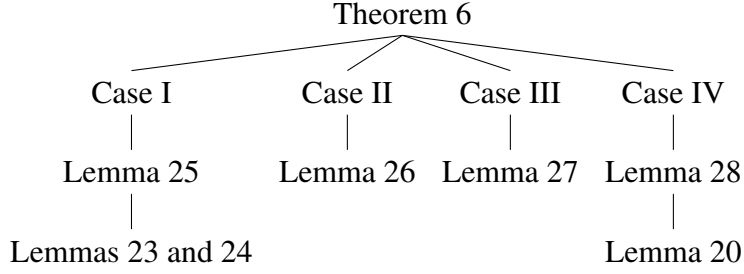
\begin{figure}
\centering
\begin{tikzpicture}[node distance=3cm]
\Tree [.{Theorem \ref{thm:perturbations}} 
         [.{Case I} 
           [.{Lemma \ref{lemma:case1NormH}} 
             {Lemmas \ref{lemma:specialCaseGNormBounded} and \ref{lemma:GeneralCaseGNormBound}}   
           ]
         ]
         [.{Case II} 
            {Lemma \ref{lemma:case2NormH}}
         ]  
         [.{Case III} 
            [.{Lemma \ref{lemma:case3NormH}} 
            ]
         ]
         [.{Case IV}  
            [.{Lemma \ref{lemma:case4normH}} 
                {Lemma \ref{lemma:Hnormbounded}}
            ]
         ]  
      ]  
\end{tikzpicture}       
\caption{Proof structure diagram for Theorem \ref{thm:perturbations}}
\label{fig:Thm6ProofTree}
\end{figure}

The proof structure of Theorem \ref{thm:perturbations} is given in Figure \ref{fig:Thm6ProofTree}. We first obtain some spectral inequalities that can be applied to graphs in either $G$ or $H$ spaces. Next, we focus on spectral inequalities on graphs and their line graphs before assembling the proof from different Lemmas as shown in Figure \ref{fig:Thm6ProofTree}.

\subsection{Spectral inequalities for generic perturbed graphs}
In this section we use graphs $F_1$ and $F_2$ where $F_2$ is a perturbed version of $F_1$. We get spectral radius bounds for different cases.  

\begin{lemma}\label{lemma:Hnormbounded}Suppose $F_1$ and $F_2$ are  connected graphs such that they differ by an edge addition or deletion, i.e., either $F_2 = \Adde(F_1)$ or $F_2 = \Dele(F_1)$.  Then
\begin{equation}\label{eq:Hbound}
    0 < C_1(F_1, F_2) \leq  \left | \lVert F_2 \rVert - \lVert F_1 \rVert \right| \leq C_2(F_1, F_2) \leq 1 \, , 
\end{equation}
where $C_1(F_1, F_2)$ and $C_2(F_1, F_2)$ depends on the principle eigenvectors of $F_1$ and $F_2$.  If  $F_2 = \Adde(F_1)$ then
\[
 0 < C_1(F_1, F_2) \leq   \lVert F_2 \rVert - \lVert F_1 \rVert \leq C_2(F_1, F_2) \leq 1 \, , 
\]
\end{lemma}
\begin{proof}
    We consider the case $F_2 = \Adde(F_1)$. Let $\bm{x}$ and $\bm{w}$ be the normalized principal eigen vectors of $A(F_2)$ and $A(F_1)$.  From the Perron-Frobenius theorem all components of $\bm{x}$ and $\bm{w}$ are positive.  Suppose the vertices $i$ and $j$ form the additional edge in $F_2$. Then from Lemma 1 and Corollary 1 in \cite{li2012bounds} we have 
    \[
   0 <  2w_iw_j \leq \lambda_1(A(F_2)) - \lambda_1(A(F_1)) \leq 2 x_ix_j \leq 1 \, . 
    \]
    By letting $C_1(F_1, F_2) =2x_ix_j$ and  $C_2(F_1, F_2) = 2 w_iw_j $ we get the result.  The case $F_1 = \Adde(F_2)$ is similar. 
\end{proof}

\begin{lemma}\label{lemma:FEdgeRelocationNorm} Suppose $F_1$ and $F_2$ are connected graphs such that they differ by an edge relocation, i.e.,  $F_2 = \relocatee(F_1)$.  Then
    \[
    0  \leq  \left | \lVert F_2 \rVert - \lVert F_1 \rVert \right| \leq C_2(F_1, F_{12}, F_2) \leq 1 \, , 
    \]
    where $F_{12}$ denotes the graph in-between $F_1$ and $F_2$, i.e., $F_{12} = \Adde(F_1)$ and $F_2 = \Dele(F_{12})$. 
\end{lemma}
\begin{proof}
    From Lemma \ref{lemma:Hnormbounded} we have 
    \begin{align}
          0 < C_1(F_1, F_{12}) & \leq   \lVert F_{12} \rVert - \lVert F_1 \rVert  \leq C_2(F_1, F_{12}) \leq 1 \, ,  \label{eq:align1}  \\
          0 < C_1(F_2, F_{12}) & \leq   \lVert F_{12} \rVert - \lVert F_2 \rVert  \leq C_2(F_2, F_{12}) \leq 1 \, ,  \label{eq:align2}\\
         -1 \leq -C_2(F_2, F_{12}) & \leq  \lVert F_2 \rVert - \lVert F_{12} \rVert \leq -C_1(F_2, F_{12}) \leq 0 \notag \\
          \lVert F_2 \rVert - \lVert F_{1} \rVert & \leq C_2(F_1, F_{12}) - C_1(F_2, F_{12}) \leq 1 \notag\\
         0  \leq  &  \left | \lVert F_2 \rVert - \lVert F_1 \rVert \right| \leq C_2(F_1, F_{12}, F_2) \leq 1 \notag
    \end{align}
    where we have multiplied equation \eqref{eq:align2} by -1 and added to equation \eqref{eq:align1} and taken the absolute value. 
\end{proof}

\begin{lemma}\label{lemma:FVertexMerge} 
    Suppose $F_1$ and $F_2$ are connected graphs such that they differ by a vertex merge or a vertex split, i.e.,  $F_2 = \mergev(F_1)$ or $F_2 = \splitv(F_1)$.  Then
    \[
    0  \leq  \left | \lVert F_2 \rVert - \lVert F_1 \rVert \right| \leq C_2(F_1, F_{12}, F_2) \leq 1 \, , 
    \]
    where $F_{12}$ denotes the graph in-between $F_1$ and $F_2$, i.e., in the case of vertex merging $F_{12} = \Adde(F_1)$ and $F_2 = \Dele(F_{12}) + \Delv(F_{12})$.
\end{lemma}
\begin{proof}
    We consider vertex merging as by relabelling $F_1$ to $F_2$ we get vertex splitting. Recall that $\mergev(F) = \Adde(F) + \Dele(F) + \Delv(F)$ (Definition \ref{def:MergeVertex}). From Lemma \ref{lemma:Hnormbounded} we have
  \begin{equation} \label{eq:f1f12}
       0 < C_1(F_1, F_{12}) \leq   \lVert F_{12} \rVert - \lVert F_1 \rVert \leq C_2(F_1, F_{12}) \leq 1 \, , 
  \end{equation}
Let $\bm{x}$ be the normalized principle eigen vector of $A(F_{12})$ and suppose we delete vertex $i$ and the incident edge from $F_{12}$ to obtain  $F_2$.  Then from Theorem 1 in \cite{li2012bounds} 
\[
(1 - 2 x_i^2) \lambda_1\left(A(F_{12})\right) \leq \lambda_1\left( A(F_2) \right) \leq \lambda_1\left( A(F_{12}) \right) \, . 
\]
This gives us
\[
- 2 x_i^2 \lVert F_{12} \rVert  \leq \lVert F_{2} \rVert  -  \lVert F_{12} \rVert  \leq 0 \, . 
\]
Adding to equation \eqref{eq:f1f12} and taking absolute values we get the result. 
\end{proof}

\subsection{Spectral inequalities for graphs and their line graphs}

\begin{lemma}\textbf{(Special case: triangle closing)}\label{lemma:specialCaseGNormBounded}
    Consider the \textit{special case: triangle closing} scenario shown in Figure \ref{fig:SpecialCase} where $H_2 = \Adde(H_1)$ and $G_2 = \relocatee(G_1)$. Then 
 
    \[
     \frac{ \left \vert  \lVert G_2 \rVert_n -  \lVert G_1 \rVert_n \right \vert }{C(G_1, G_2, G_{int}) } \leq \frac{\left \vert  \lVert H_2 \rVert_m -  \lVert H_1 \rVert_m \right \vert }{C_1(H_1, H_2)} \leq 1 \, , 
    \]       
    where $G_{int}$ denotes an intermediate graph between $G_1$ to $G_2$, where $G_{int} = \Adde(G_1)$ and $G_2 = \Dele(G_{int})$ and $C(G_1, G_2, G_{int})$ depends on the normalized principal eigenvectors of $G_1$, $G_2$ and $G_{int}$. 
\end{lemma}
\begin{proof}
    Applying Lemma \ref{lemma:Hnormbounded} to $H_1$ and $H_2$ we have 
     \[
     \frac{1}{C_2(H_1, H_2)} \leq \frac{1}{ \left | \lVert H_2 \rVert_m - \lVert H_1 \rVert_m \right|} \leq \frac{1}{C_1(H_1, H_2)} \, . 
     \]
     Applying Lemma \ref{lemma:FEdgeRelocationNorm} we have 
     \[
     0 \leq \left | \lVert G_{2} \rVert_n - \lVert G_1 \rVert_n \right| \leq C(G_{int}, G_1, G_2)
     \]
    which gives us
     \[
    \frac{ \left | \lVert G_{2} \rVert_n - \lVert G_1 \rVert_n \right|}{ \left | \lVert H_2 \rVert_m - \lVert H_1 \rVert_m \right|} \leq \frac{C(G_{int}, G_1, G_2)}{C_1(H_1, H_2)} \, . 
     \]
Reorganising terms and recognising $C(H_1, H_2) \leq 1$ (equation \eqref{eq:Hbound}) gives the result.    

\end{proof}

\begin{lemma}(\textbf{General case})\label{lemma:GeneralCaseGNormBound}
    Consider the \textit{general case} scenario shown in Figure \ref{fig:GeneralCase} where $H_2 = \Adde(H_1)$ and $G_2 = \mergev(G_1)$. Then 
    \[
     \frac{ \left \vert  \lVert G_2 \rVert_n -  \lVert G_1 \rVert_n \right \vert }{C(G_1, G_2, G_{int}) } \leq \frac{\left \vert  \lVert H_2 \rVert_m -  \lVert H_1 \rVert_m \right \vert }{C_1(H_1, H_2)} \leq 1 \, , 
    \]  
    where $G_{int}$ denotes an intermediate graph between $G_1$ to $G_2$, where $G_{int} = \Adde(G_1)$ and $G_2 = \Dele(G_{int})$ and $C(G_1, G_2, G_{int})$ depends on the normalized principal eigenvectors of $G_1$, $G_2$ and $G_{int}$. 
    
\end{lemma}
\begin{proof}
Applying Lemma \ref{lemma:Hnormbounded} to $H_1$ and $H_2$ we get
 \[
     \frac{1}{ \left | \lVert H_2 \rVert_m - \lVert H_1 \rVert_m \right|} \leq \frac{1}{C_1(H_1, H_2)} \, .
     \]
Applyling Lemma \ref{lemma:FVertexMerge} to $G_1$ and $G_2$ we get
\[
\left \vert \lVert G_2 \rVert_n - \lVert G_1 \rVert_n \right \vert \leq C(G_1, G_{2}, G_{int}) \, . 
\]
Multiplying the two inequalities we get
 \[
     \frac{  \left \vert \lVert G_2 \rVert_n - \lVert G_1 \rVert_n \right \vert }{ \left | \lVert H_2 \rVert_m - \lVert H_1 \rVert_m \right|} \leq \frac{C(G_1, G_{2}, G_{int})}{C_1(H_1, H_2)} \, , 
     \]
 which can be reorganised and combined with Lemma \ref{lemma:Hnormbounded}  to obtain the result.

\end{proof}



\subsection{Assembling the proof of Theorem \ref{thm:perturbations}}
\begin{lemma}\textbf{(Case I)} \label{lemma:case1NormH} 
For edge augmented $H$ (Definition \ref{def:edegaugmentedH}) if  $\widetilde{H} \cong \hat{H}$ then
 \[
     \frac{ \left \vert  \lVert \hat{G}\rVert_n -  \lVert G \rVert_n \right \vert }{C(G, \hat{G}, G_{int}) } \leq \frac{\left \vert  \lVert \hat{H} \rVert_m -  \lVert {H} \rVert_m \right \vert }{C_1(H, \hat{H})} \leq 1 \, , 
 \]
 where $G_{int}$ denotes an intermediate graph between $G$ and $\hat{G}$.
\end{lemma} 
\begin{proof}
    When $\widetilde{H} \cong \hat{H}$ from  Lemma \ref{lemma:case1AboutG} either the special case or the general case describe modifications in the $G$ space. Thus, the result follows from Lemmas \ref{lemma:specialCaseGNormBounded} and \ref{lemma:GeneralCaseGNormBound}.
\end{proof}

\begin{lemma}(\textbf{Case II})\label{lemma:case2NormH} For edge augmented $H$ (Definition \ref{def:edegaugmentedH}) 
suppose $\hat{H} = \Dele(\widetilde{H})$. Then 
\[ e_1 \cong e_2  \iff  \lVert \hat{H} \rVert = \lVert H \rVert \iff \lVert  \hat{G} \rVert  = \lVert  G \rVert. 
\]    
\end{lemma}
\begin{proof}
  The result  follows from Lemma \ref{lemma:case2}.
\end{proof}

\begin{lemma}(\textbf{Case III})\label{lemma:case3NormH} For edge augmented $H$ (Definition \ref{def:edegaugmentedH})
$\hat{H} = L(\hat{G}) = \Dele(\widetilde{H}) = \widetilde{H} - e_2$. Then if $e_1 \ncong e_2$ 
\[
0 \leq \left \vert  \lVert \hat{H} \rVert_m -  \lVert {H} \rVert_m \right \vert \leq C(H, \hat{H}, \widetilde{H}) \leq 1 \, , 
\]
and
\[
0 \leq \left \vert  \lVert \hat{G}\rVert_n -  \lVert G \rVert_n \right \vert  \leq C(G, \hat{G}, G_{int_1}, \ldots, G_{int_4} ) \leq 2\, . 
\]
 where  $G_{int_1} \ldots, G_{int_1}$  denote different intermediate graphs between $G$ and $\hat{G}$. Furthermore, if $\left \vert  \lVert \hat{H} \rVert_m -  \lVert {H} \rVert_m \right \vert \geq C > 0$, where $C = C(H, \hat{H}, \widetilde{H})$ then 
\[
     \frac{ \left \vert  \lVert \hat{G}\rVert_n -  \lVert G \rVert_n \right \vert }{C(G, \hat{G}, G_{int_1} \ldots, G_{int_4}) } \leq \frac{\left \vert  \lVert \hat{H} \rVert_m -  \lVert {H} \rVert_m \right \vert }{C_1(H, \hat{H}, \widetilde{H})}   \, , 
 \]

\end{lemma}
\begin{proof}
As $e_1 \ncong e_2$ we have $\hat{H} = \relocatee(H)$ and from Lemma \ref{lemma:FEdgeRelocationNorm} we have
\[
0 \leq \left \vert  \lVert \hat{H} \rVert_m -  \lVert {H} \rVert_m \right \vert \leq C(H, \hat{H}, \widetilde{H}) \leq 1 \, .
\]
For Case III (relocate edge) from Lemma \ref{lemma:case3second} either $\hat{G} = \relocatee(G)$ or $\hat{G} = \mergev(G) + \splitv(G)$. For $\hat{G} = \relocatee(G)$ from Lemma \ref{lemma:FEdgeRelocationNorm}
\[
0 \leq \left \vert  \lVert \hat{G} \rVert_n -  \lVert {G} \rVert_n \right \vert \leq C(G, \hat{G}, G_{int_1}) \leq 1 \, ,
\]
where $G_{int_1}$ is the intermediate graph between $G$ and $\hat{G}$ in this case. For $\hat{G} = \mergev(G) + \splitv(G)$, let us consider an intermediate graph $G_{int_2} = \mergev(G)$. Then from Lemma \ref{lemma:FVertexMerge}
\[
0 \leq \left \vert  \lVert {G}_{int_2} \rVert_n -  \lVert {G} \rVert_n \right \vert \leq C(G, G_{int_{2}},  G_{int_{3}}) \leq 1 \, ,
\]
where $G_{int_{3}}$ is the intermediate graph that occurs when merging a vertex as discussed in Lemma \ref{lemma:FVertexMerge}. Similarly, considering $\hat{G} = \splitv(G_{int_2})$ we get
\[
0 \leq \left \vert  \lVert {G}_{int_2} \rVert_n -  \lVert \hat{G} \rVert_n \right \vert \leq C(\hat{G}, G_{int_{2}},  G_{int_{4}}) \leq 1 \, ,
\]
where $G_{int_{4}}$ is another intermediate graph. Combining the inequalities for $\hat{G} = \mergev(G) + \splitv(G)$ we get
\[
0 \leq \left \vert  \lVert \hat{G}\rVert_n -  \lVert G \rVert_n \right \vert  \leq C(G, \hat{G}, G_{int_1}, \ldots, G_{int_4} ) \leq 2\, . 
\]
If $\left \vert  \lVert \hat{H} \rVert_m -  \lVert {H} \rVert_m \right \vert \geq C > 0$, we have 
\[
\frac{1}{\left \vert  \lVert \hat{H} \rVert_m -  \lVert {H} \rVert_m \right \vert} \leq \frac{1}{C} \, , 
\]
giving us 
\[
\frac{\left \vert  \lVert \hat{G}\rVert_n -  \lVert G \rVert_n \right \vert }{\left \vert  \lVert \hat{H} \rVert_m -  \lVert {H} \rVert_m \right \vert} \leq \frac{C(G, \hat{G}, G_{int_1}, \ldots, G_{int_4} )}{C}  \, .
\]
\end{proof}

\begin{lemma}(\textbf{Case IV})\label{lemma:case4normH} For edge augmented $H$ (Definition \ref{def:edegaugmentedH}) 
suppose $\hat{H} = L(\hat{G}) = \Adde(\widetilde{H}) = \widetilde{H} + e_2$. Then 
\[
0 < C_1(H, \widetilde{H}, \hat{H}) \leq  \lVert \hat{H}   \rVert_m  - \lVert H   \rVert_m \leq C_2(H, \widetilde{H}, \hat{H}) \leq 2 \, . 
\]
\end{lemma}
\begin{proof}
 As $\widetilde{H} = \Adde(H)$ from Lemma \ref{lemma:Hnormbounded} we have
 \[
 0 < C_1(H, \widetilde{H}) \leq  \lVert \widetilde{H}   \rVert_m  - \lVert H   \rVert_m \leq C_2(H, \widetilde{H}) \leq 1 \, . 
 \]
 As $\hat{H} = \Adde(\widetilde{H})$ again from Lemma \ref{lemma:Hnormbounded} we get
 \[
 0 < C_1(\hat{H}, \widetilde{H}) \leq  \lVert   \hat{H} \rVert_m  - \lVert \widetilde{H}   \rVert_m \leq C_2(\hat{H}, \widetilde{H}) \leq 1 \, . 
 \]
 Adding these inequalities we get the result. 
\end{proof}

\thmperturbations*
\begin{proof}
    Case I is proved in Lemma \ref{lemma:case1NormH}, Case II in Lemma \ref{lemma:case2NormH}, Case III in Lemma \ref{lemma:case3NormH}  and Case IV in Lemma \ref{lemma:case4normH}.
\end{proof}
\section{Adding an edge to a line graph: an example}\label{sec:addedgetoLG}
Let $G$ be a graph and $H = L(G)$ be its line graph. Suppose $i, j \in V(H)$ but there is no edge connecting $i$ and $j$. Then suppose we add an edge connecting vertices $i$  and $j$ to $H$ and call the resulting graph $\widetilde{H}$. Is $\widetilde{H}$ a line graph? This depends on vertices $i$ and $j$. 



From \cite{krausz1943demonstration} we know that the edges of $H$  can be partitioned in to complete subgraphs in such a away that no vertex belongs to more than two of the subgraphs. Let $S_1$ be the set of vertices in $H$ that belong to only one complete subgraph and let $S_2$ be the set of vertices that belong to two complete subgraphs. 

Let us consider the graph $G$ and its line graph $H = L(G)$ in Figure \ref{fig:agraphanditslinegraph}. 
\begin{figure}[!ht]
    \centering
    \includegraphics[width=0.8\linewidth]{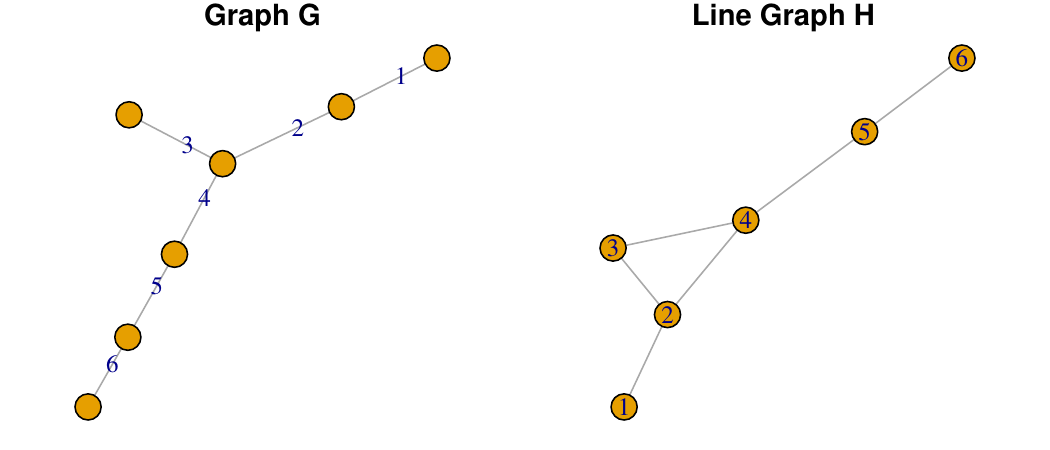}
    \caption{A graph $G$ and its line graph $H = L(G)$.}
    \label{fig:agraphanditslinegraph}
\end{figure}

\subsection{When both $i, j \in S_1$}

In line graph $H$ as shown in Figure \ref{fig:agraphanditslinegraph} the complete subgraphs are $\{ \{1, 2\}, \{2, 3, 4 \},  \{4, 5\}, \{5,6\} \}$. Thus the set of vertices belonging to a single complete subgraph $S_1 = \{1, 3, 6\}$ and the set of vertices belonging to two complete subgraphs  $S_2 = \{2, 4, 5\}$.  When we consider both $i,j \in S_1$ we have 3 choices for the edge $i$-$j$: 1-3, 3-6, 1-6. 

Figure \ref{fig:addedEdgetoH1} shows $\widetilde{H}$ and $L^{-1}(\widetilde{H})$ when the edge 1-3 is added. We see that $G$ (in Figure \ref{fig:agraphanditslinegraph}) had 7 vertices, but $L^{-1}(\widetilde{H})$ has 6 vertices. Adding the edge 1-3 to the line graph $H$ merged two vertices as seen in $L^{-1}(\widetilde{H})$.
\begin{figure}
    \centering
    \includegraphics[width=0.8\linewidth]{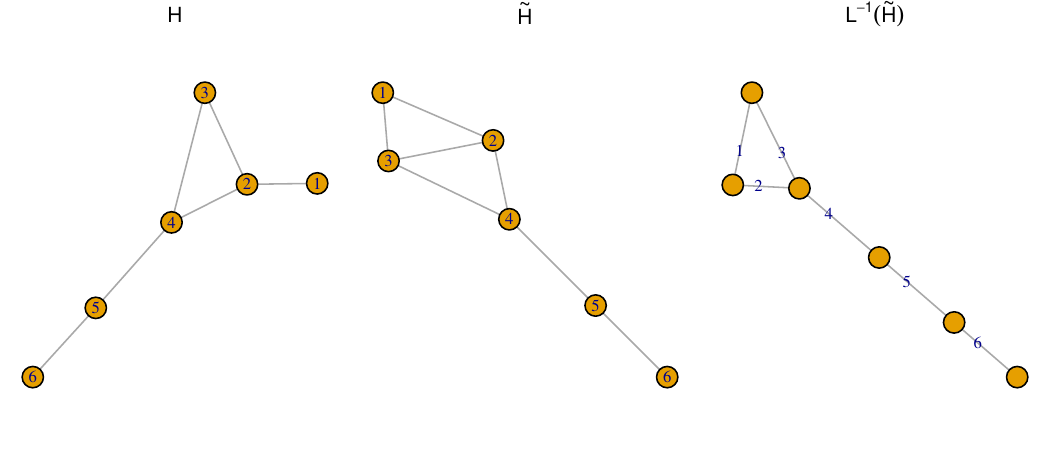}
    \caption{ Edge 1-3 added to $H$ resulting in $\widetilde{H}$. In this case the inverse line graph $L^{-1}(\widetilde{H})$ exists. } 
    \label{fig:addedEdgetoH1}
\end{figure}
Similarly, if we add the edge 3-6 or 1-6 to $H$, we still would be able to find $L^{-1}(\widetilde{H})$. 

\subsection{When $i \in S_1$ and $j \in S_2$}
In this case we can add one of the following edges to $H$: 1-4, 1-5, 3-5, 2-6, 4-6. If we add edge 1-4 to $H$, this results in the induced subgraph $K_{1,3}$($L_1$ in Figure \ref{fig:ForbiddenGraphs}), which is forbidden, on vertices $\{1, 4, 5, 3 \}$. Thus, 1-4 is not a valid additional edge. Adding edges 1-5 and 2-6 give rise to the induced subgraph $K_{1,3}$ as well.  

Adding the edge 3-5 to $H$ would result in the forbidden subgraph $L_4$, which doesn't have an inverse line graph. Adding the edge 4-6 is permissible and would result in an $\widetilde{H}$ with an $L^{-1}(\widetilde{H})$ as shown in Figure \ref{fig:addedEdgetoH2}. 

\begin{figure}[!ht]
    \centering
    \includegraphics[width=0.8\linewidth]{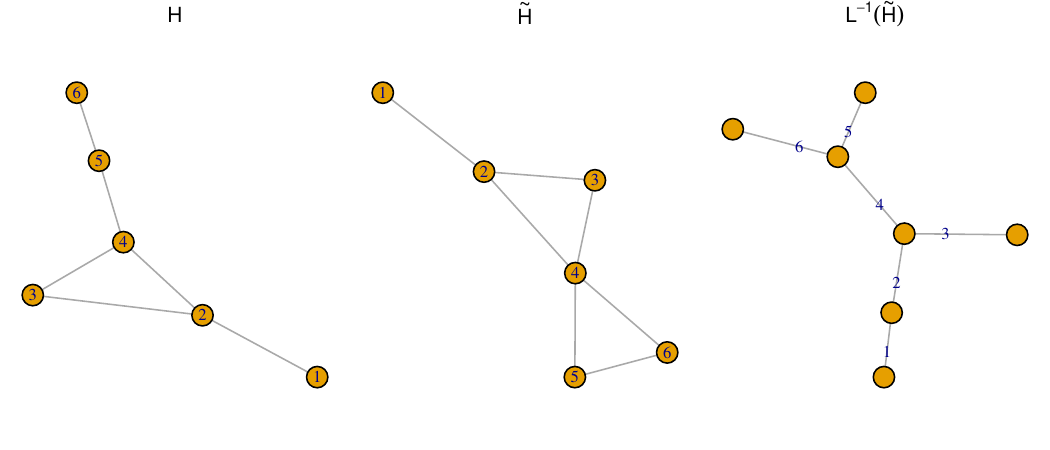}
    \caption{Edge 4-6 added to $H$ resulting in $\widetilde{H}$. The inverse line graph $L^{-1}(\widetilde{H})$ exists as shown. }
    \label{fig:addedEdgetoH2}
\end{figure}

\subsection{When $i, j \in S_2$}
As $S_2 = \{2,4,5\}$ the only option for both $i,j$ to be in $S_2$ is to add new edge 2-5 as both edges 2-4 and 4-5 are existing edges in $H$.  But adding 2-5 would induce $K_{1,3}$ on vertices $\{1,2,5,3\}$ in $H$. Thus, this is not a permissible edge, in the sense $L^{-1}(\widetilde{H})$ does not exist in this case.

\section{Adding line-forbidden graphs to $H$}
We conducted 2 experiments. We considered a graph $G$ and its line graph $H = L(G)$. Then we modified $H$ by adding either $L_2$ or $L_5$ to $H$, where $L_2$ and $L_5$ are line forbidden graphs illustrated in Figure \ref{fig:ForbiddenGraphs}. Using the modified graph $\widetilde{H}$ we computed $L^{\dagger}(\widetilde{H})$.

\subsection{Adding $L_2$ to $H$}
In the first experiment we considered graphs from the Barabási–Albert (BA) model \citep{barabasi1999emergence}. We generated graphs of $n$ vertices with $n$ ranging from 10 to 20. For each $n$ we generated 5 graphs to account for randomisation. We computed the line graph $H = L(G)$ for each graph $G$. Then we added the forbidden line graph $L_2$ to $H$ as follows: first we considered the disjoint union of $L_2$ and $H$, then we merged one of $L_2$'s vertices with a vertex from $H$. Figure \ref{fig:L2andHtilde} shows two orientations of $L_2$ and an example of $L_2$ merged with $H$ producing $\tilde{H}$. The vertices in subgraph $L_2$ in $\widetilde{H}$ are coloured in blue. The R package \texttt{igraph} \citep{igraphR} was used to generate the graphs. Figure \ref{fig:algorithmExample1} gives an example of $\hat{H} = L(L^{\dagger}(\widetilde{H}))$.

\begin{figure}[t]
   \subfloat[\label{fig:L2}]{%
      \includegraphics[width=0.40\textwidth]{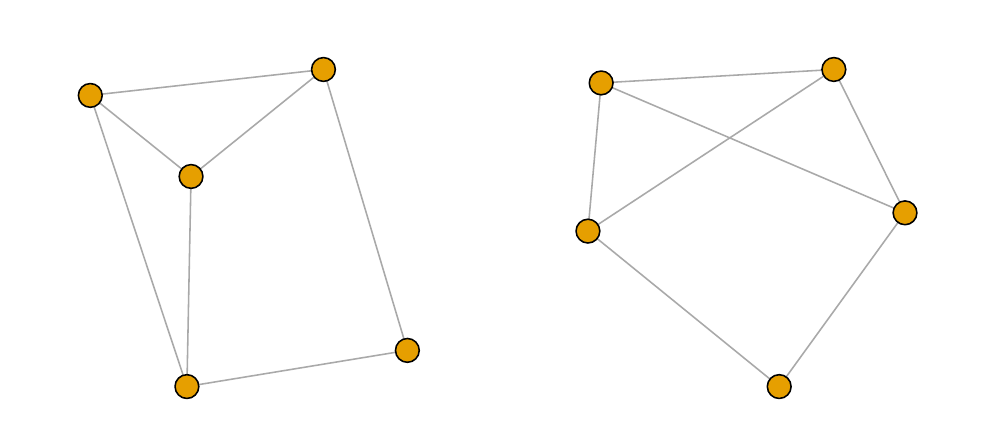}
    }
    \hfill
    \subfloat[\label{fig:HandHtilde}]{%
      \includegraphics[width=0.55\textwidth]{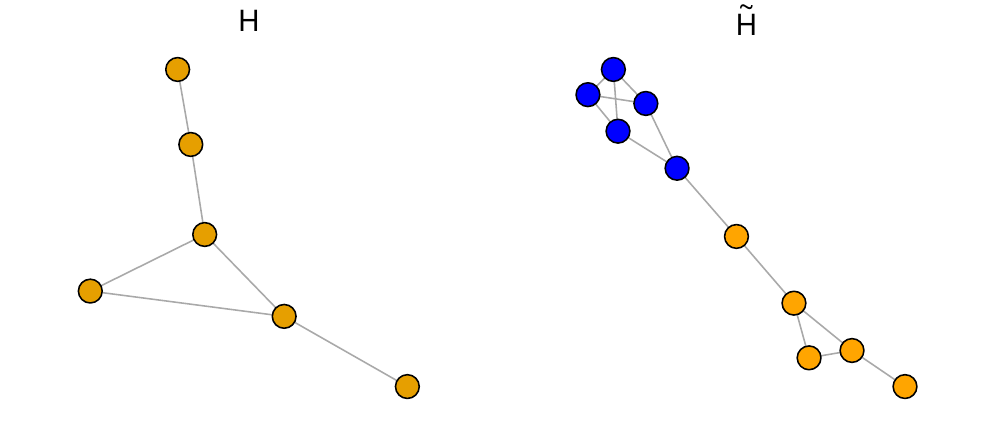}
    }
   \caption{(a) The forbidden line graph $L_2$ in 2 different orientations. (b) Graph $H$ and $\widetilde{H}$ obtained by merging $L_2$ with $H$. Vertices from $L_2$ are shown in blue. }
\label{fig:L2andHtilde}
\end{figure}

\begin{figure}[t]
    \centering
    \includegraphics[width=0.9\linewidth]{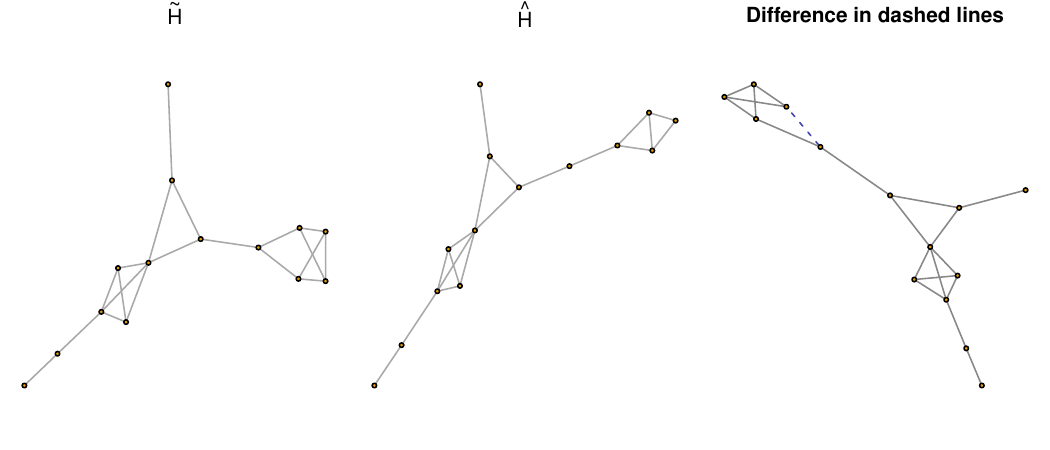}
    \caption{Graph $\widetilde{H}$ on the left, $\hat{H}$ in the middle and the difference in dashed lines on the right. }
    \label{fig:algorithmExample1}
\end{figure}

\begin{table}[t]
    \centering
    \caption{Experiment results}
    \label{tab:L2L5added}
    \begin{tabular}{cccc}
      \toprule
 Exp. & Edge edits    & Add & Remove \\
    \midrule
\multirow{2}{*}{1}  &  1 &    3 & 34    \\
                    & 2  &  0 & 21 \\
\midrule
\multirow{2}{*}{2}  &  1   &   3 & 26  \\
                    &  2   &  0 & 29 \\
   \bottomrule
    \end{tabular}

\end{table}



\subsection{Adding $L_5$ to $H$}
Similar to Experiment 1, we added $L_5$ in Figure \ref{fig:ForbiddenGraphs} to $H$ and computed a pseudo-inverse. Table \ref{tab:L2L5added} gives the results of these two experiments. There were 55 graphs in each set with experiment 1 having $\widetilde{H} = H + L_2$ and experiment 2 satisfying $\widetilde{H} = H + L_5$. As given in Table \ref{tab:L2L5added}, for experiment 1,  for 3 graphs, a single edge was added and for 34 graphs a single edge was removed to obtain a pseudo-inverse. For 21 graphs in experiment 1, 2 edges were removed to obtain a pseudo-inverse. We see that depending on the position where $L_2$ was added to the graph, the number of edge edits change.  Similarly for experiment 2, for 3 graphs an edge was added, for 26 graphs an edge was removed and for 29 graphs 2 edges were removed to obtain a pseudo-inverse.

\section{Estimating the size of the haplotype pool} \label{sec:AppExperiments}

In population studies estimating the number of ancestors of a given population is an important task. This is generally done by genotyping a sample of individuals at some single neucleotide polymorphism (SNP) locations  and phasing the genotype data to predict the underlying haplotypes \citep{halldorsson2013estimating}. An alternative approach using inverse line graphs was first proposed by \cite{halldorsson2013estimating} and then by \cite{labbe2021finding}. Effectively, they find a pseudo-inverse line graph $\hat{G}$ to a given Clark Consistency graph $\widetilde{H}$. However, they only consider edge deletions as their method is targeted to population estimation in haplotype phasing. Our method is slightly more general as we consider edge additions, however it is not targeted to population estimation in haplotype phasing. 

We use the dataset provided by \cite{ferdosi2014hsphase} to validate our method for population estimation. The dataset is a genotype matrix $A$ where columns represent SNPs and rows represent genotypes. The entries of $A$ denoted by $a_{ij}$  are codified $\{0, 1, 2\}$. If both corresponding parent haplotype SNPs are coded 1, then $a_{ij} = 2$, if both parent haplotype SNPs are coded 0, then $a_{ij} = 0$, otherwise $a_{ij} = 1$. Two genotypes have a common ancestor if their SNP strings are consistent where consistency is defined as follows: given an SNP either one of the strings have code 1, or both strings have the same code 0 or 2 \citep{labbe2021finding}. For example, the two strings $s_1 = 201$ and $s_2 = 101$ are consistent because in the first and third position $s_2$ has 1s, and in the second position both $s_1$ and $s_2$ agree. The Clark Consistency graph is constructed from a set of genotypes by denoting each genotype as a node and connecting two nodes by an edge if the two genotypes are consistent. The Clark Consistency graph is \textit{allelable} if it is line-invertible and the number of nodes in the inverse line graph is an estimate for the ancestor population size. In instances where the Clark Consistency graph is not a line graph, the number of nodes in a computed pseudo-inverse is considered the ancestor population size estimate.

We consider two 100 sample datasets, one with 10 genotypes  and the other with 15, recorded at 100 SNPs.  A genotype gives an individual's genetic makeup, specifically referring to the set of alleles that an organism inherits from its parents. Each sample dataset gives us an $n \times 100$ matrix where $n=10$ or $n=15$. From this matrix we construct the Clark Consistency graph $\widetilde{H}$ and find a pseudo-inverse $\hat{G}$. The estimated haplotype population size is $|V(\hat{G})|$. 

We validate our results using haplotype phasing methods discussed in \cite{ferdosi2013effect}. Using these methods we compute the size of the single parent haplotype population pool and the size of the haplotype population pool when both parents are considered.  The size of the single parent haplotype population pool is a lower bound for the haplotype population size because it only takes the haplotypes of one parent into account. Similarly, the size of the population taking into account both parents is an upper bound because some parents' alleles are not present in their offspring. 
Thus, we have lower and upper bound estimates for the haplotype population size. Figure \ref{fig:haplotypeExample} shows the results with the red curve showing the upper bound, the blue curve showing the lower bound and the green curve showing the estimate from the pseudo-inverse method.

\end{document}